\documentclass{article} 

\usepackage[final]{neurips_2022}

\usepackage[utf8]{inputenc} 
\usepackage[T1]{fontenc}    
\usepackage[colorlinks=true,linkcolor=blue,citecolor=green,urlcolor=blue]{hyperref}
\usepackage{url}            
\usepackage{booktabs}       
\usepackage{amsfonts}       
\usepackage{nicefrac}       
\usepackage{microtype}      
\usepackage{xcolor}         

\usepackage{mathtools}
\usepackage{amssymb}
\usepackage{latexsym}
\usepackage{dsfont}
\usepackage{mathrsfs}
\usepackage{amssymb}
\usepackage{amsfonts}
\usepackage{bm}
\usepackage{xspace}
\usepackage{amsthm}
\usepackage{color}
\usepackage{booktabs}

\newcommand*{\argmax}{\mathop{\mathrm{argmax}}}

\newcommand{\Acal}{\mathcal{A}}
\newcommand{\Bcal}{\mathcal{B}}

\newcommand{\Dcal}{\mathcal{D}}

\newcommand{\Ncal}{\mathcal{N}}

\newcommand{\Tcal}{\mathcal{T}}

\newcommand{\Xcal}{\mathcal{X}}
\newcommand{\Ycal}{\mathcal{Y}}

\newcommand{\Ebb}{\mathbb{E}}

\newcommand{\Pbb}{\mathbb{P}}

\ifx\BlackBox\undefined
\newcommand{\BlackBox}{\rule{1.5ex}{1.5ex}}  %
\fi

\ifx\QED\undefined
\def\QED{~\rule[-1pt]{5pt}{5pt}\par\medskip}
\fi

\ifx\proof\undefined
\newenvironment{proof}{\par\noindent{\em Proof:\ }}{\hfill\BlackBox\\[.0mm]}
\fi
\ifx\theorem\undefined
\newtheorem{theorem}{Theorem}[section]
\fi

\ifx\example\undefined
\newtheorem{example}{Example}[section]
\fi

\ifx\property\undefined
\newtheorem{property}{Property}
\fi

\ifx\lemma\undefined
\newtheorem{lemma}{Lemma}[section]
\fi

\ifx\proposition\undefined
\newtheorem{proposition}{Proposition}[section]
\fi

\ifx\remark\undefined
\newtheorem{remark}{Remark}
\fi

\ifx\corollary\undefined
\newtheorem{corollary}{Corollary}[section]
\fi

\ifx\definition\undefined
\newtheorem{definition}{Definition}[section]
\fi

\ifx\conjecture\undefined
\newtheorem{conjecture}{Conjecture}[section]
\fi

\ifx\axiom\undefined
\newtheorem{axiom}[theorem]{Axiom}
\fi

\ifx\claim\undefined
\newtheorem{claim}[theorem]{Claim}
\fi

\ifx\assumption\undefined
\newtheorem{assumption}{Assumption}
\fi

\ifx\problem\undefined
\newtheorem{problem}{Problem}[section]
\fi

\ifx\fact\undefined
\newtheorem{fact}{Fact}
\fi

\newcommand{\bx}{\bm{x}}
\newcommand{\by}{\bm{y}}

\newcommand{\facc}[2]{{#1}{\scriptsize±{#2}}}
\usepackage{paralist}
\usepackage{floatrow}
\usepackage{wrapfig,lipsum,booktabs}
\usepackage{caption}
\usepackage{array}
\usepackage{subcaption}
\usepackage{soul}
\usepackage{xcolor}
\usepackage{multicol}
\usepackage{multirow}
\usepackage{caption}
\usepackage{subcaption}
\usepackage{graphicx}
\usepackage[usestackEOL]{stackengine}

\title{Transferring Fairness under Distribution Shifts \\ via Fair Consistency Regularization}

\author{%
  Bang An\\
  Department of Computer Science\\
  University of Maryland, College Park\\
  \texttt{bangan@umd.edu} \\
   \And
   Zora Che \\
   Department of Computer Science \\
   Boston University \\
   \texttt{zche@bu.edu} \\
   \And
   Mucong Ding \\
   Department of Computer Science\\
  University of Maryland, College Park\\
  \texttt{mcding@umd.edu} \\
  \And
  Furong Huang\\
  Department of Computer Science\\
  University of Maryland, College Park\\
  \texttt{furongh@umd.edu} 
}

\begin{document}

\maketitle

\begin{abstract}
The increasing reliance on ML models in high-stakes tasks has raised a major concern about fairness violations. Although there has been a surge of work that improves algorithmic fairness, most are under the assumption of an identical training and test distribution. In many real-world applications, however, such an assumption is often violated as previously trained fair models are often deployed in a different environment, and the fairness of such models has been observed to collapse. In this paper, we study how to transfer model fairness under distribution shifts, a widespread issue in practice. We conduct a fine-grained analysis of how the fair model is affected under different types of distribution shifts and find that domain shifts are more challenging than subpopulation shifts. Inspired by the success of self-training in transferring accuracy under domain shifts, we derive a sufficient condition for transferring group fairness. Guided by it, we propose a practical algorithm with fair consistency regularization as the key component. A synthetic dataset benchmark, which covers diverse types of distribution shifts, is deployed for experimental verification of the theoretical findings. Experiments on synthetic and real datasets, including image and tabular data, demonstrate that our approach effectively transfers fairness and accuracy under various types of distribution shifts\footnote{Code is available at https://github.com/umd-huang-lab/transfer-fairness.}.
\end{abstract}

\setlength\abovedisplayskip{2pt}
\setlength\belowdisplayskip{2pt}

\section{Introduction}\label{sec:intro}

\begin{figure}[!htbp]
\vspace{-2em}
  \begin{minipage}[c]{0.55\textwidth}
    \includegraphics[width=\textwidth]{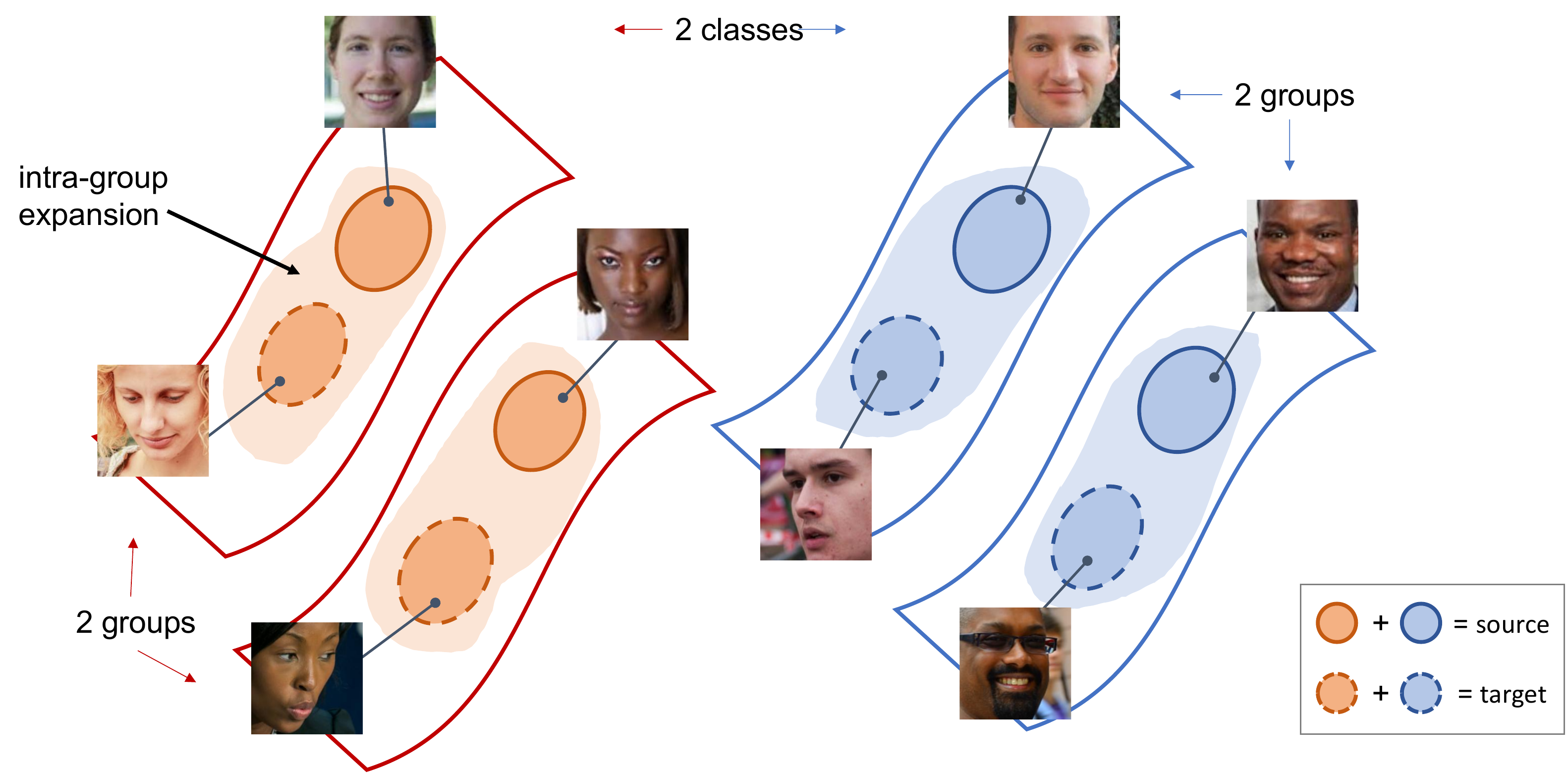}
  \end{minipage}\hfill
  \begin{minipage}[c]{0.44\textwidth}
    \begin{center}
    \resizebox{\textwidth}{!}{
    \begin{tabular}{c|cc}
    \toprule
         & \textbf{Source} & \textbf{Target} \\
         \midrule
       Sampled from  & UTKFace~\cite{zhifei2017cvpr} & FairFace~\cite{karkkainen2019fairface}\\
       \midrule
       Class & \multicolumn{2}{c}{Gender} \\
       \midrule
       Sensitive attribute & \multicolumn{2}{c}{Race}\\
       \midrule
       Group  & \multicolumn{2}{c}{Gender $+$ Race} \\
       \midrule
        Shared  & \multicolumn{2}{c}{Underlying data generation process} \\
        \midrule
        \multirow{2}{*}{Differ} &  \multicolumn{2}{c}{Capturing-bias, e.g., diff}\\
        & \multicolumn{2}{c}{angles, facial expressions} \\ 
        \bottomrule
    \end{tabular}
    }
    \end{center}
  \end{minipage}
  \vspace{-0.5em}
  \caption{\textbf{Illustration of intra-group expansion assumption in the input space.} An example of gender classification task with the sensitive attribute being race.  Intra-group expansion assumes that different groups are separated 
  but every group is self-connected under certain transformations. If a model has consistent predictions under those transformations, we can propagate labels within each group. Under this assumption, we propose to obtain fairness and accuracy in both domains by a self-training algorithm with fair consistency regularization.}\label{fig:overview}
\end{figure}

Machine learning's social impact has broadened as it is widely used to aid decision-making in real-world applications, such as hiring, loan approval, facial recognition, and criminal justice. 
To avoid discrimination against a subset of the population (e.g., w.r.t race or gender), many efforts on algorithmic fairness have been carried out \cite{chouldechova2016fair, friedler2016impossibility, pmlr-v54-zafar17a, survey, rajkomar2018ensuring, corbettdavies2018measure, caton2020fairness}.
Although existing work has achieved remarkable success in ensuring fairness, most of them assume the distribution of data at test time is identical to that in the training set.
However, recent studies show that the fairness of a model is likely to collapse when encountering a distribution shift. 
For example, \cite{ding2021retiring} observes that a fair income predictor trained with data from one state might not be fair when used in other states.
\cite{schrouff2022maintaining} tries to maintain fairness in healthcare settings, but a model that performs fairly according to the metric evaluated in ``Hospital A'' shows unfairness when applied to ``Hospital B''. 
Such observations motivate us to find the reason behind the collapse of fairness and investigate how to transfer fairness under distribution shifts. 
Specifically, when we have labeled data in the source domain and unlabeled data in the target domain, we investigate how to adapt the fair source model to a target domain with the goal of achieving both accuracy and fairness in both domains.

Intuitively, the fairness of a model in the target domain strongly depends on the nature of distribution shifts. 
In this paper, we only consider cases where the oracle model is the same in two domains.
We characterize distribution shifts by assuming two domains share the same underlying data generation process where data is generated from a set of latent factors with a fixed generative model, and the shift is caused by the shift of the marginal distribution of some factors.
We categorize distribution shifts into three types \cite{koh2021wilds}:
1) \textit{Domain shift} 
where source and target distributions comprise data from related but distinct domains (e.g., train a model in hospital A but test it in hospital B). 2) \textit{Subpopulation shift} where two domains overlap, but relative proportions of subpopulations differ (e.g., the proportion of female candidates increases at test time). 3) \textit{Hybrid shift} where domain shift and subpopulation shift happen at the same time.
We find domain shift more challenging for transferring fairness since the model's performance is unpredictable in unseen domains. 
Such a finding is supported empirically on a synthetic dataset that is developed to simulate diverse types of distribution shifts.
While recent work explores methods to transfer fairness \cite{singh_violate, rezaei_robust, giguere2022fairness}, most considered settings fall into subpopulation shifts. In this paper, we consider all three types of distribution shifts. 
Our analysis suggests we encourage consistent fairness under different factor values.

We draw inspiration from recent progress on self-training in transferring accuracy under domain shifts~\cite{colin2021, tianle2021, zhang2021semisupervised, berthelot2021adamatch, Sagawa2021ExtendingTW, sohn2020fixmatch}. 
The success of self-training is due to an \textit{expansion assumption} and a \textit{consistency regularization} algorithm. 
The expansion assumption also assumes two domains share one underlying generative model and the support of the distribution on each class is a connected compact set under data transformations (i.e., has a good continuity). 
Under the \textit{expansion assumption}, \cite{colin2021} and \cite{tianle2021}
prove that self-training, which enforces consistent predictions for the same input under different transformations (i.e., under shifts of nuisance factors), can propagate labels from the source to the target domain.
This approach exhibits superior performance in transferring accuracy \cite{zhang2021semisupervised, Sagawa2021ExtendingTW}, but does not consider fairness.

Taking demography into consideration, we relax the expansion assumption to a more realistic \textit{intra-group expansion assumption}, as shown in Figure~\ref{fig:overview}, which only requires continuity of the underlying distribution within every group (i.e., data points with the same class and sensitive attribute) rather than the entire class.
Based on the intra-group expansion assumption, we derive a sufficient condition that guarantees fairness in both source and target domains. 
This sufficient condition suggests that ensuring the trained model gains the same consistency across groups under a fair teacher classifier guarantees fairness in both domains.
However, such a teacher classifier is not available in practice, and we need a practical treatment.

Guided by the theoretical algorithm, we propose a practical self-training algorithm to minimize and balance consistency loss across groups. Our algorithm builds upon Laftr \cite{david_learning_fair}, an adversarial learning method for fairness, and FixMatch \cite{sohn2020fixmatch}, a self-training framework. 
To encourage similar consistency in different groups, we propose a novel \textit{fair consistency regularization}. 
By reweighting the consistency loss of each group dynamically according to the model's performance, the algorithm encourages the model to pay more attention to the high-error group while training. Our method results in a model that is fair in source and has similar consistency across groups. As indicated by our theory, it would have similar accuracy across groups in the target domain so that we can transfer fairness.
We evaluate our method under different types of distribution shifts with the synthetic and real datasets. Experiments show that our approach achieves high accuracy and fairness in the target domain without sacrificing performance in the source domain. 
To the best of our knowledge, this is the first work using self-training to transfer fairness under distribution shifts.

\textbf{Summary of contributions:}
\textbf{(1)} We provide a fine-grained analysis of fairness under distribution shifts and develop a synthetic dataset to study model fairness under different types of distribution shifts.
\textbf{(2)} Theoretically, we derive a sufficient condition for transferring fairness under distribution shifts.
\textbf{(3)} Algorithmically, we propose a theory-guided algorithm for transferring fairness with a fair consistency regularization as the key component. 
\textbf{(4)} Experimentally, we evaluate our method on synthetic data, real image data, and real tabular data. All results show the effectiveness of our approach in transferring fairness.

\section{Preliminaries and Notations}\label{sec:prelim}

\textbf{Transfer Fairness.}
Let $X, A, Y$ and $\Xcal, \Acal, \Ycal$ denote random variables and sample space of input features, sensitive attribute, and label.
For simplicity, we assume binary sensitive attribute and binary classification, while our method can easily extend to multi-sensitive attributes and multi-class cases (see Appendix \ref{app:add-exp}).
We aim to learn a classifier $g: \Xcal \rightarrow \Ycal$ and are interested in its fairness under distribution shifts. 
Specifically, with $S$ and $T$ denoting source and target domains, we study how to transfer fairness and accuracy when $\Pbb_S(X, A, Y) \neq \Pbb_T(X, A, Y)$, with the
access to $X, A, Y$ in the source domain, but only $X, A$ in the target domain.
In the self-training algorithm, we use $g_{tc}$ to denote a teacher classifier, and $g^*$ to denote the oracle classifier. 
We use the word ``group'' to denote the set of data that has the same label and sensitive attribute. 

\textbf{Fairness Metric.} Since we consider classification problems in this paper, we expect the fairness metrics could encourage models to achieve similar classification performance across groups. We use two metrics in this paper, \textit{equalized odds} and \textit{variance of group accuracy}. \textit{Equalized odds} \cite{hardt2016equality} is a widely used unfairness metric in classification problems that requires the true positive rate and the true negative rate to be the same among groups. It is defined as 
$\Delta_{odds}= \frac{1}{2}\sum_{y=0}^1\big|\Pbb(\hat{Y}=y | A=0, Y=y) - \Pbb(\hat{Y}=y | A=1, Y=y)\big|$, 
where $\hat{Y}=g(X)$ is the prediction. 
Additionally, we also evaluate the \textit{variance of group accuracy} which is defined as $V_{acc}=Var(\{\Pbb(\hat{Y}=y|A=a, Y=y), \forall a,y\})$. Smaller $V_{acc}$ indicates the model is fairer since it performs similarly across groups. 
Note that the variance of group accuracy can help avoid trivial fairness where a model with constant output has $\Delta_{odds}=0$, but such fairness is meaningless.

\section{Fairness under Distribution Shifts}\label{sec:shift}
In this section, we provide a fine-grained analysis of fairness under various types of distribution shifts based on a unified framework of distribution shift characterization.

\textbf{A Unified Framework to Characterize Distribution Shifts.} 
Following \cite{wiles2022a}, we characterize distribution shifts by assuming a unified latent variable model for the underlying data generation process. 
We denote the underlying factors as $Y^1, Y^2, ..., Y^K$, and data point as $X$. 
Two of the factors are label $Y^l$ (i.e. $Y$) and sensitive attribute $Y^a$ (i.e. $A$). We call other factors \textit{nuisance factors} since they are irrelevant to the classification task. 
\begin{assumption}\label{asum:under}
(Underlying data generation process) We assume the data is generated from a latent generative model as $\by^{1:K}\sim \Pbb(Y^{1:K})$ and $\bx \sim \Pbb(X|Y^{1:K}=\by^{1:K})$. The generative model is fixed $\Pbb_S(X|Y^{1:K}=\by^{1:K})=\Pbb_T(X|Y^{1:K}=\by^{1:K})$ but the marginal distribution of factors varies in two domains $\Pbb_S(Y^{1:K})\neq\Pbb_T(Y^{1:K})$, causing the distribution shift $\Pbb_S(Y^{1:K},X)\neq \Pbb_T(Y^{1:K},X)$.

\end{assumption}
It is realistic to assume two domains share the same data generation process. For example, the underlying physical process of cell imaging is fixed, while the distribution of underlying factors (e.g. \textit{gender}, \textit{age} or \textit{equipment}) may vary in two hospitals (i.e. two domains), resulting in the distribution shift of the observed tissue images. Based on the unified framework, we consider two major types of distribution shifts, namely \textit{subpopulation shift} and \textit{domain shift}, which are widely considered in many practical applications \cite{koh2021wilds}.

\begin{definition}(Subpopulation shift)
We say it is a subpopulation shift, if for any factor $Y^i$, the sample space of it remains the same in two domains (i.e., $ \Ycal^i_S= \Ycal^i_T$), but the marginal distribution of at least one factor changes (e.g., $\Pbb_S(Y^{j})\neq \Pbb_T(Y^{j})$), resulting in $\Pbb_S(Y^{1:K})\neq \Pbb_T(Y^{1:K})$ and $\Pbb_S(Y^{1:K},X)\neq \Pbb_T(Y^{1:K},X)$. 
\end{definition}
\begin{definition}(Domain shift) We say it is a domain shift, if at least one nuisance factor $Y^i, i\neq l, i\neq a$, has different sample space in two domains, $\exists y^i\in \Ycal^i_T$, but $y^i \notin \Ycal^i_S$, resulting in $\Pbb_S(Y^{1:K})\neq \Pbb_T(Y^{1:K})$ and $\Pbb_S(Y^{1:K},X)\neq \Pbb_T(Y^{1:K},X)$.
\end{definition}
Intuitively, under subpopulation shift, the sample space overlaps, and only the marginal distributions of factors vary in the two domains. 
For example, the proportion of females versus males in training and deployment time differs. 
In contrast, under domain shift, the source model has never seen the data with factor values that only exist in the target domain. For instance, the source model is unaware of the equipment used for cell imaging at deployment time. 

\textbf{Why do distribution shifts cause unfairness?} Suppose the marginal distributions of a binary nuisance factor $Y^i$ differ in two domains with $\Pbb_S(Y^i)\neq \Pbb_T(Y^i)$. The unfairness in two domains are
\begin{align}
    \Delta_{odds}^S&=\Pbb_S(Y^i=0) \times \Delta_{odds}^S|_{Y^i=0} + \Pbb_S(Y^i=1) \times \Delta_{odds}^S|_{Y^i=1} \label{eq:unfair}\\
    \Delta_{odds}^T&=\Pbb_T(Y^i=0) \times \Delta_{odds}^T|_{Y^i=0} + \Pbb_T(Y^i=1) \times \Delta_{odds}^T|_{Y^i=1}. \notag
\end{align}
Due to the same generation process where $\Pbb_S(X|Y^i=y^i)=\Pbb_T(X|Y^i=y^i)$, we have $\Delta_{odds}^S|_{Y^i=y^i}=\Delta_{odds}^T|_{Y^i=y^i}$, $\forall y^i\in\{0,1\}$. 
Under subpopulation shift, $Y^i$ has the same sample space in two domains but with different proportions (e.g., ${\Pbb_S(Y^i=0)=0.9}, {\Pbb_S(Y^i=1)=0.1}, \Pbb_T(Y^i=0)=0.1, {\Pbb_T(Y^i=1)=}0.9$), while under domain shift the sample space differs (e.g. $\Pbb_S(Y^i=0)=1, \Pbb_T(Y^i=1)=1$). 
It is easy to see from~\eqref{eq:unfair} that if a model is perfectly fair on data with $Y^i=0$ but unfair on data with $Y^i=1$, then the model is highly fair in the source domain but highly unfair in the target domain under both cases. 
Therefore, if the model has inconsistent performance on data generated from different nuisance factor values, then the shifted marginal distribution of those factors may cause fairness collapse.

\textbf{How to transfer fairness under distribution shifts?} 
Based on the above analysis, one way is to train the model to be fair under any values of factors. It is possible under subpopulation shift as stated in the following proposition (see proof and discussion in Appendix \ref{app:shift}).
\begin{proposition}\label{prop:subshift}
(Transfer fairness under subpopulation shift)
Consider the subpopulation shift that is caused by the shifted marginal distribution of nuisance factor $Y^i$ (i.e., $\Pbb_S(Y^i)\neq \Pbb_T(Y^i)$), while $\Ycal_S^i=\Ycal_T^i=\Ycal^i$. If model $f$ is strictly fair in source domain under any value of factor $Y^i$ satisfying
$\Pbb_S(g(X)=y^l | Y^a=0, Y^l=y^l, Y^i=y^i)=\Pbb_S(g(X)=y^l | Y^a=1, Y^l=y^l, Y^i=y^i), \forall y^i\in \Ycal^i, y^l\in\{0,1\}$, 
then model $g$ is also fair in target domain with $\Delta_{odds}=0$.
\end{proposition}
Our empirical results (Figure~\ref{fig:3dshapes}) also support this finding.
However, domain shift is more challenging. The source model's performance on target data is unpredictable due to the distinct sample space. 
One promising way to tackle domain shift is to enforce the model's invariance to nuisance factors so that the source model would have the same behavior on target data.
Note that this solution also works for subpopulation shift since it leads to the case in Proposition~\ref{prop:subshift} directly. 
The above analysis motivates us to transfer fairness by encouraging consistent fairness under different nuisance factor values.

\section{Transfer Fairness via Fair Consistency Regularization}
\looseness=-1

\subsection{Theoretical Analysis: A Sufficient Condition for Transferring Fairness}
In reality, distribution shifts are usually hybrid, and we may not know all the underlying factor values. In this section, we consider a general case where we only have access to input $X$, label $Y$, and sensitive attribute $A$. We use data transformations to simulate the shift of nuisance factors. 
Our theory is based on \cite{colin2021} and \cite{tianle2021} which prove that encouraging consistency under transformations can propagate labels so that to transfer accuracy. In this section, we find that in order to transfer fairness, we need a fair label propagation process that requires the model to have similar consistency across groups. 
We introduce assumptions and our findings as follows.

\begin{assumption}
[Separability of the input] 
Let $S_a^y$ and $T_a^y$ denote the sample space of $X|_{A= a, Y=y}$ in source and target domains. The ground truth class and sensitive attribute for $\bx\in S_a^y \cup T_a^y$ are consistent, which are $y\in\{0,1\}$ and $a\in\{0,1\}$.
We assume the sample spaces of $X$ in two domains are $S = \cup_{y}\cup_{a}S_a^y$ and $T = \cup_{y}\cup_{a}T_a^y$, where groups are separated with
\begin{inparaenum}[1)]
\item $S_a^y \cap S_{a'}^y = T_a^y \cap T_{a'}^y = S_a^y \cap T_{a'}^y = \emptyset, \forall y, a \neq a'$, and
\item $S_a^y \cap S_{a'}^{y'} = T_a^y \cap T_{a'}^{y'} = S_a^y \cap T_{a'}^{y'} = \emptyset, \forall a, a', y \neq y'$. 
\end{inparaenum}
\end{assumption}\label{asum:data}

This is a realistic assumption as illustrated in Figure~\ref{fig:overview} where the data from two domains are from the same underlying conditional distribution $X|_{Y,A}$, and groups are separated by label and sensitive attribute.
We define $U_a^y = \frac{1}{2}(S_a^y+T_a^y)$ as the group distribution, and $U$ as the population distribution on the entire data.
Next, we characterize the good continuity of group distributions with the definition of \textit{neighbor} and \textit{intra-group expansion} assumption.
\begin{definition} 
[Neighbor] Let $\Tcal$ denote a set of input transformations and define the transformation set of $\bx$ as $\Bcal(\bx)\triangleq\{\bx'| \exists t \in \Tcal, \textnormal{s.t. } \|\bx'-t(\bx)\|\leq r\}$. For any $\bx\in S_a^y \cup T_a^y$, we define the neighbor of $\bx$ as
$\Ncal(\bx) := (S_a^y \cup T_a^y) \cap \{\bx' | \Bcal(\bx)\cap \Bcal(\bx') \neq \emptyset\}$ and define the neighbor of a set $V\in \Xcal$ as $\Ncal(V) := \cup_{\bx\in V \cap (\cup_{y}\cup_{a}S_a^y \cup T_a^y)}\Ncal(x)$.
\end{definition}\label{def:trans}
Intuitively, two examples are neighbors if they are near each other after applying some transformations.
Note that we only consider neighbors that have the same class and sensitive attribute (i.e., from the same group). 
Based on this definition, we characterize the continuity of group distribution with \textit{intra-group expansion} assumption where any small set has a large neighbor in its group.
\begin{assumption}
[Intra-group expansion] We say that $U_a^y$ satisfies $(\alpha, c)$-multiplicative expansion for some constant $\alpha\in (0,1)$ and $c>1$, if for all $V \subset U_a^y$ with $\Pbb_{U_a^y}(V)\leq \alpha$, the following holds: 
\begin{align*}
    \Pbb_{U_a^y}(\Ncal(V))\geq \min\{c\Pbb_{U_a^y}(V),1\}.
\end{align*}
\end{assumption}
Different from the \textit{expansion} assumption proposed in \cite{colin2021} which considers the class continuity, \textit{intra-group expansion} assumes group continuity.
As shown in Figure~\ref{fig:overview}, this is more realistic since groups are separated by both label and sensitive attribute. 
We can also interpret it as the transformations that change the value of nuisance factors will generate neighbors within the same group.

This assumption allows us to propagate labels within the group from one domain to another by encouraging consistency under transformations. 
We use $R_{U_a^y}(g) \triangleq \Pbb _{U_a^y}[\exists\bx'\in \Bcal(\bx), \textnormal{s.t. } g(\bx)\neq g(\bx')]$ to denote the \textit{consistency loss} of classifier $g$ on the group distribution $U_a^y$, which is the fraction of examples where $g$ is not robust to input transformations.
Since we only have partial supervision (i.e., no labels in the target domain), we use a self-training framework to obtain a model that is accurate and fair in both domains (i.e., on $U_a^y$). 
Based on the theory of self-training in \cite{colin2021}, we derive a sufficient condition in Theorem~\ref{th:main} that bounds the unfairness and error on the population distribution. 
We use 0-1 loss to evaluate the \textit{error} of $g$ as $\varepsilon_{U_a^y}(g) \triangleq \Pbb_{U_a^y}[g(\bx)\neq g^*(\bx')]$, and the \textit{disagreement} between $g$ and a teacher classifier $g_{tc}$ as $ L_{U_a^y}(g, g_{tc})\triangleq \Pbb_{U_a^y}[g(\bx)\neq g_{tc}(\bx')]$.

\begin{theorem} 
[Guarantee fairness]
Suppose we have a teacher classifier $g_{tc}$ with bounded unfairness such that $|\varepsilon_{U_a^y}(g_{tc})-\varepsilon_{U_{a'}^{y'}}(g_{tc})|\leq \gamma, \forall a,a'\in \Acal$ and $y,y'\in \Ycal$.
We assume intra-group expansion where $U_a^y$ satisfies $(\bar{\alpha}, \bar{c})$-multiplicative expansion and $\varepsilon_{U_a^y}(g_{tc})\leq\bar{\alpha}<1/3$ and $\bar{c}>3, \forall a,y$. We define $c\triangleq\min\{1/\bar{\alpha}, \bar{c}\}$, and set $\mu \leq \varepsilon_{U_a^y}(g_{tc}), \forall a,y$.
If we train our classifier with the algorithm
\begin{align}
    \min_{g\in G}\max_{a, y} R_{U_a^y}(g), \quad\quad
    \textnormal{s.t. } \quad L_{U_a^y}(g, g_{tc})\leq \mu \quad \forall a,y \notag
\end{align}
then the error and unfairness of the optimal solution $\hat{g}$ on the distribution $U$ are bounded with
\begin{align}
    \varepsilon(\hat{g}) &\leq \frac{2}{c-1} \varepsilon_U(g_{tc}) + \frac{2c}{c-1}R_U(\hat{g}), \\
    \Delta_{odds}(\hat{g}) &\leq \frac{2}{c-1}(\gamma + \mu + c\max_{a, y} R_{U_a^y}(\hat{g})) 
\end{align}
\end{theorem}\label{th:main}
\textbf{Remark.} This sufficient condition suggests we fit a teacher classifier which is fair on the population distribution and minimize the \textit{consistency loss} in every group. 
The unfairness of the resulting model is bounded by the quality (unfairness and error) of the teacher classifier and the worst-group consistency loss.
Intuitively, we can understand the consistency loss as the model invariance to the nuisance factors.
With a group-balanced consistency loss, the model would have similar invariance to the nuisance factors resulting in similar group performance on the unseen data so that to transfer accuracy and fairness. 
We also bound the variance of group accuracy with the variance of consistency loss (Appendix \ref{app:proof}). Both bounds suggest we balance and minimize the consistency loss across groups.

\vspace{-1em}
\subsection{Practical Algorithm: Fair Consistency Regularization}
\vspace{-1em}
There are two challenges in realizing the theoretical algorithm in Theorem~\ref{th:main}. First, we need a high-quality teacher model, but the model trained with labeled source data is only fair and accurate in the source domain. 
Second, existing consistency regularization methods do not consider fairness.
We tackle the first problem by leveraging the iterative self-training paradigm that updates the teacher model with the student model while training, thus making it fairer and fairer.
We tackle the second problem by proposing a novel fair consistency regularization.

\begin{wrapfigure}{r}{0.59\textwidth}
\captionsetup{font=footnotesize}
    \centering
    \includegraphics[width=\textwidth]{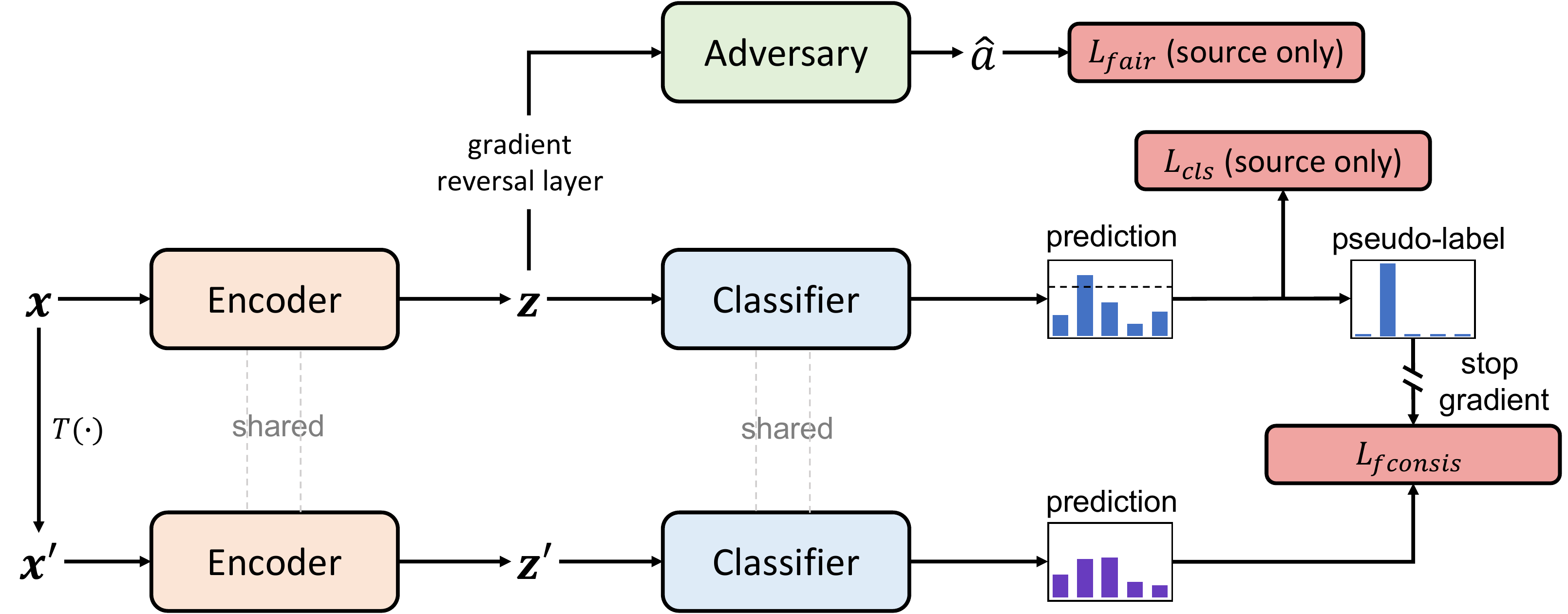}
    \caption{\small{Training diagram.}}\label{fig:archi}
\end{wrapfigure}
\textbf{Algorithm.}
Figure~\ref{fig:archi} shows the overall training diagram.
There are three major components:\\
\textbf{(1)} In every training epoch, we use the student model obtained in the last epoch as the teacher model and automatically fit the teacher model by initializing the student model to be the same as the teacher model. In other words, only one model is training itself iteratively.\\
\textbf{(2)} To ensure the accuracy and fairness in the source domain, we adopt Laftr \cite{david_learning_fair}, an adversarial learning method consisting of a classification loss $L_{cls}$ and a fairness loss $L_{fair}$.\\
\textbf{(3)} To transfer fairness and accuracy, we do consistency training on all unlabeled data (including source and target data). Following FixMatch \cite{sohn2020fixmatch}, we use the pseudo-labels generated by the teacher model as supervision for consistency training where the model should have consistent predictions under transformations. Different from FixMatch, we propose a fair consistency regularization with a balanced group consistency loss $L_{fconsis}$. \\
We train the model with the weighted summation of these three losses as shown in Figure~\ref{fig:archi}. We defer the detailed loss functions of $L_{cls}$ and $L_{fair}$ with a detailed algorithm description to Appendix \ref{app:exp}.

\textbf{Fair Consistency Regularization.}
To tighten the upper bound of the unfairness in Theorem \ref{th:main}, we need to minimize and balance consistency loss across groups. 
However, the consistency regularization in FixMatch \cite{sohn2020fixmatch} does not distinguish groups and might amplify the bias as observed in \cite{zhu2022the} and our experiments.
Instead, we propose to use a fair consistency regularization that evaluates the consistency loss per group and minimizes the balanced consistency loss $L_{fconsis}$ defined as below.
\begin{align}
\setlength{\abovedisplayskip}{-2pt}
\setlength{\belowdisplayskip}{-2pt}
    &L_{fconsis}(g) = \sum_{y=0}^1 \sum_{a=0}^1 \lambda_a^y L_a^y(g)\\
    \text{where}\quad &L_a^y(g)=\frac{1}{\sum_{\bx_a^y}\mathds{1}}\sum_{\bx_a^y} \mathds{1}(\max(g_{tc}(\bx_a^y))\geq \tau)H(\argmax(g_{tc}(\bx_a^y)), g(t(\bx_a^y))) \label{eq:L_consis}
\end{align}
where $\bx_a^y$ denotes an input with sensitive attribute $A=a$ and class $Y=y$.
$L_a^y(g)$ is model $g$'s consistency in the group of $\{\bx_a^y\}$,
and $\lambda_a^y$ is the corresponding weight of the group consistency loss.
Here, we abuse the usage of $g(\bx)$ to denote the output logits of model $g$ on input $\bx$ and thus, $\argmax(g_{tc}(\bx_a^i))$ is the pseudolabel generated by teacher classifier. $t(\bx_a^y)$ is the transformed input as defined in Definition~\ref{def:trans}. 
We use a cross-entropy loss $H(\cdot)$ to encourage the consistency under transformation $t(\cdot)$ and only consider examples that the teacher model has high confidence in with a confidence threshold $\tau$. 
Note that data is classified into groups according to the true sensitive attribute and pseudolabels. 
To balance the group consistency loss, we propose to weigh each group inversely with the number of confident pseudolabels, and set $\lambda_a^y$ as
\begin{align}\label{eq:coeff}
    \hat{\lambda}_a^y = \frac{1}{\sum_{\bx_a^y}\mathds{1}(\max(g_{tc}(\bx_a^y))\geq \tau)}, \quad \lambda_a^y=\hat{\lambda}_a^y/\sum_{a,y}\hat{\lambda}_a^y.
\end{align}
The weights will dynamically change while training. 
Heuristically, if the teacher model is only confident in a few examples in a group, the model's consistency in this group is more likely to be low.
With the proposed weights, a larger penalty will be applied to such groups.
Therefore, the proposed fair consistency regularization will enforce the model to pay more attention to high-error groups. 
By doing so, the trained model would enjoy similar consistency loss across groups. Together with the self-training algorithm, it would have similar accuracy across groups in the target domain.

\section{Related Work}
\vspace{-0.5em}
This section features related work for transferring fairness. Another discussion of related work in fair machine learning, domain adaptation, and self-training is deferred to Appendix~\ref{app:sec:related}.
Out-of-distribution fairness remains an under-explored area. We categorize prior works into five classes. 
\begin{inparaenum}[1)]
    \item \textit{Group-wise distribution matching}. 
\cite{candice_transfer} derives an upper bound for fairness in the target domain which suggests training a fair model in the source domain and matching the distributions of relevant groups from two domains in feature space at the same time. \cite{yoon_wasser} also applies group-wise distribution matching but with Wasserstein distance. Such methods are hard to achieve if we do not have supervision in the target domain and it also shares the drawback of distribution matching methods.
\item \textit{Reweighting}. When the proportions of groups differ in two domains, reweighting the examples in the source domain can approximate the target distribution.  \cite{coston_missing} uses reweighting to deal with fairness problems under covariate shift and \cite{giguere2022fairness} uses reweighting together with a fairness test to guarantee fairness under demographic shift. Reweighting methods strongly rely on the support cover assumption which is not satisfied under domain shift.
\item \textit{Distributionally robust optimization (DRO)}. This line of work considers unknown target data that can be any arbitrary weighted combinations of the source dataset and train a fair model that is robust to the worst-case shift \cite{rezaei_robust, mandal2020ensuring}. These methods also assume subpopulation shift instead of domain shift.
\item \textit{Causal inference}. \cite{singh_violate} conducts causal domain adaptation and DRO based on a well-characterized causal graph that describes the data construction and distribution shift. Causal methods highly rely on the correct causal graph which is hard to obtain in reality. For example, \cite{schrouff2022maintaining} finds that the causal graph in real applications (e.g. predicting the skin condition in dermatology) is far more complicated which violates normal assumptions, thus making those approaches inapplicable.
\item \textit{Others}. \cite{chen2022fairness} derives bound for fairness violation under distribution shifts. There are also studies that aim to maintain fairness under distribution shifts through online learning \cite{zhang2021farf}, and loss curvature matching \cite{wang2022equalized}.
\end{inparaenum}
To the best of our knowledge, this is the first work that uses self-training to transfer fairness. Some work also studies self-supervised learning and fairness, yet they use unlabeled data and self-training to improve the in-distribution fairness \cite{chzhen2019leveraging, 9117188, chakraborty2021can} which is different from our goal.

\section{Experiments}
\begin{figure}
    \centering
    \includegraphics[width=\textwidth]{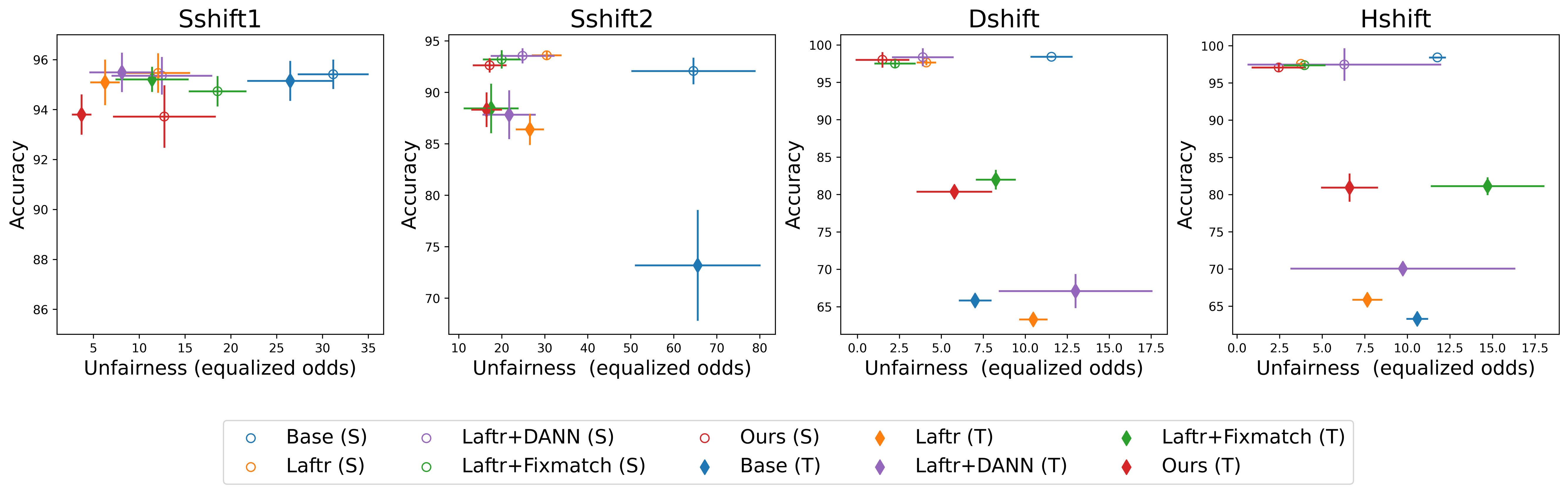}
    \vspace{-1.8em}
    \caption{Accuracy and unfairness (error bar denotes the standard deviation) in two domains under subpopulation shifts (Sshift 1, Sshift 2), domain shift (Dshift), and hybrid shift (Hshif). (S) and (T) denotes the evaluation in the source and target domains respectively. Results show that domain shift is more challenging than subpopulation shift, and our method can effectively transfer accuracy and fairness under all the distribution shifts considered. }
    \label{fig:3dshapes}
\end{figure}
\vspace{-0.5em}
\subsection{Evaluation under Different Types of Distribution Shifts with a Synthetic Dataset}\label{exp:syn}
In order to study the fairness under distribution shifts and verify our theoretical findings, we develop a synthetic dataset to simulate different types of distribution shifts. 

\textbf{Synthetic dataset.} The synthetic dataset is adapted from the 3dshapes dataset \cite{pmlr-v80-kim18b} which contains images of 3D objects generated from six independent latent factors (\textit{shape}, \textit{object hue}, \textit{scale}, \textit{orientation}, \textit{floor hue}, \textit{wall hue}). 
This dataset satisfies our assumption on the shared underlying data generation process. 
We simulate different types of distribution shifts by varying the marginal distributions of the latent factors and sample the data accordingly (see Appendix~\ref{app:3dshapes} for details).

\textbf{Distribution shifts.} 
We set the image as input $X$, and select three latent factors to be class ($Y=\textit{shape}$), sensitive attribute ($A=\textit{object hue}$), and a nuisance factor that might shift ($D=\textit{scale}$). 
We consider four widely observed distribution shifts in reality ( $\Pbb_S(X, Y,A,D)\neq \Pbb_T(X, Y,A,D)$): \\
(1) \textbf{Sshift 1}: Subpopulation shift where only the nuisance factor shift (i.e. more small objects in source but more large objects in target), $\Pbb_S(Y,A)=\Pbb_T(Y,A)$, $\Pbb_S(D)\neq \Pbb_T(D)$.\\
(2) \textbf{Sshift 2}: Subpopulation shift where $A$ and $Y$ have different correlations in two domains (i.e. most red objects are cubes in source but are capsules in target), $\Pbb_S(Y,A)\neq \Pbb_T(Y,A)$, $\Pbb_S(D)= \Pbb_T(D)$.\\
(3) \textbf{Dshift}: Domain shift where the nuisance factor has different sample spaces (i.e. only small objects in source but only large objects in target), $\Pbb_S(Y,A)=\Pbb_T(Y,A)$, $\Pbb_S(D)\neq \Pbb_T(D)$, $\Ycal^d_S \neq \Ycal^d_T$.\\
(4) \textbf{Hshift}: Hybrid shift of (2) and (3).

\textbf{Baselines.} We do shape classification task with an MLP model and compare our method with four baselines: 
Base (standard ERM); Laftr; Laftr+DANN (a combination of Laftr and a domain adaptation method \cite{ganin2016domain}); Laftr+FixMatch. 
In our method, we also use Laftr and FixMatch but with the proposed fair consistency regularization.
Since the shifted nuisance factor is \textit{scale}, we use random padding and cropping as transformations in our method and Laftr+FixMatch. We train Base and Laftr with labeled source data and train others with unlabeled target data as well.

\looseness=-1
\textbf{Domain shift is more challenging than subpopulation shift.}  Figure~\ref{fig:3dshapes} shows that under subpopulation shifts, 
the fair source model trained with Laftr also has high accuracy and fairness in the target domain although it has not seen any target date. 
This is because the sample space is shared (e.g. small and large objects both exist in the source data), and the model has similar performance under all factor values. 
Thus, good performance remains even if the proportion of data changes, verified Proposition~\ref{prop:subshift}.
In contrast, under domain shift and hybrid shift, the fair source model performs poorly in the target domain where data is sampled from a different sample space, suggesting the difficulty of domain shift.

\textbf{Our method can transfer fairness and accuracy under various types of distribution shifts.}
Under domain shift, the domain adaptation method DANN does not help in transferring fairness or accuracy. 
Consistency regularization forces the model to behave consistently under cropping and padding, resulting in a model that has similar predictions regardless of the object's scale and thus transfers accuracy. 
However, it may cause bias as shown in the results of Laftr+FixMatch. 
With the proposed fair consistency regularization, 
the model gains similar consistency across groups, resulting in a similar accuracy in all groups in the target domain and thus transfers fairness. 
Therefore, our method achieves high accuracy and fairness in two domains under all the considered distribution shifts.

\begin{table}[!htbp]
\captionsetup{font=footnotesize}
\small
\centering
\caption{Transfer fairness and accuracy from UTKFace to FairFace} 
\resizebox{0.8\columnwidth}{!}{
\begin{tabular}{llllllll}
\toprule
               & \multicolumn{3}{c}{Source}            && \multicolumn{3}{c}{Target}            \\ 
               \cline{2-4} \cline{6-8} 
                              & Acc  & \multicolumn{2}{c}{Unfairness} && Acc & \multicolumn{2}{c}{Unfairness}\\
               \cline{3-4} \cline{7-8}
Method         &    & $V_{acc}$ & $\Delta_{odds}$ &  & & $V_{acc}$ & $\Delta_{odds}$ \\
\midrule
Base           & \facc{92.85}{0.49} &  \facc{2.30}{0.97}  & \facc{4.81}{0.69} && \facc{74.49}{0.83} & \facc{5.79}{3.49} & \facc{9.90}{1.27}   \\
Laftr           & \facc{93.24}{0.41} &  \facc{1.19}{0.46}  & \facc{2.44}{0.51} && \facc{74.35}{1.46} & \facc{6.92}{0.72} & \facc{9.79}{1.54}   \\
CFair          & \facc{92.51}{0.22} &  \facc{1.76}{0.53}  & \facc{4.75}{0.85} && \facc{73.53}{0.89} & \facc{7.51}{0.73} & \facc{7.26}{1.95}   \\
\midrule
Laftr+DANN      & \facc{91.33}{0.08} &  \facc{2.12}{1.72}  & \facc{2.70}{0.67} && \facc{74.28}{1.63} & \facc{6.25}{2.59} & \facc{8.27}{2.11}   \\
CFair+DANN      & \facc{90.89}{0.76} &  \facc{2.01}{0.70}  & \facc{4.43}{1.36} && \facc{74.62}{1.06} & \facc{6.23}{0.90} & \facc{5.26}{2.07}   \\
Laftr+FixMatch  & \facc{96.62}{0.06} &  \facc{0.77}{0.21}  & \facc{2.23}{0.44} && \facc{83.87}{0.48} & \facc{8.21}{0.67} & \facc{9.32}{1.01}   \\
CFair+FixMatch  & \facc{96.13}{0.53} &  \facc{1.28}{0.53}  & \facc{2.78}{0.74} && \facc{83.11}{0.49} & \facc{7.87}{1.86} & \facc{7.89}{0.40}   \\
Ours (w/ Laftr)  & \facc{96.08}{0.07} &  \facc{0.96}{0.39}  & \facc{2.59}{0.35} && \facc{85.52}{0.40} & \facc{2.82}{0.87} & \facc{5.70}{0.52}   \\
Ours (w/ CFair)  & \facc{95.65}{0.22} &  \facc{1.56}{0.37}  & \facc{3.85}{0.97} && \facc{84.48}{0.42} & \facc{2.88}{0.99} & \facc{5.43}{0.65}   \\
\bottomrule
\end{tabular}\label{tab:face-vgg_rand}
}
\end{table}

\begin{figure}[h]
\captionsetup{font=footnotesize}
    \centering
    \begin{subfigure}[b]{0.7\textwidth}
        \centering
         \includegraphics[width=\textwidth]{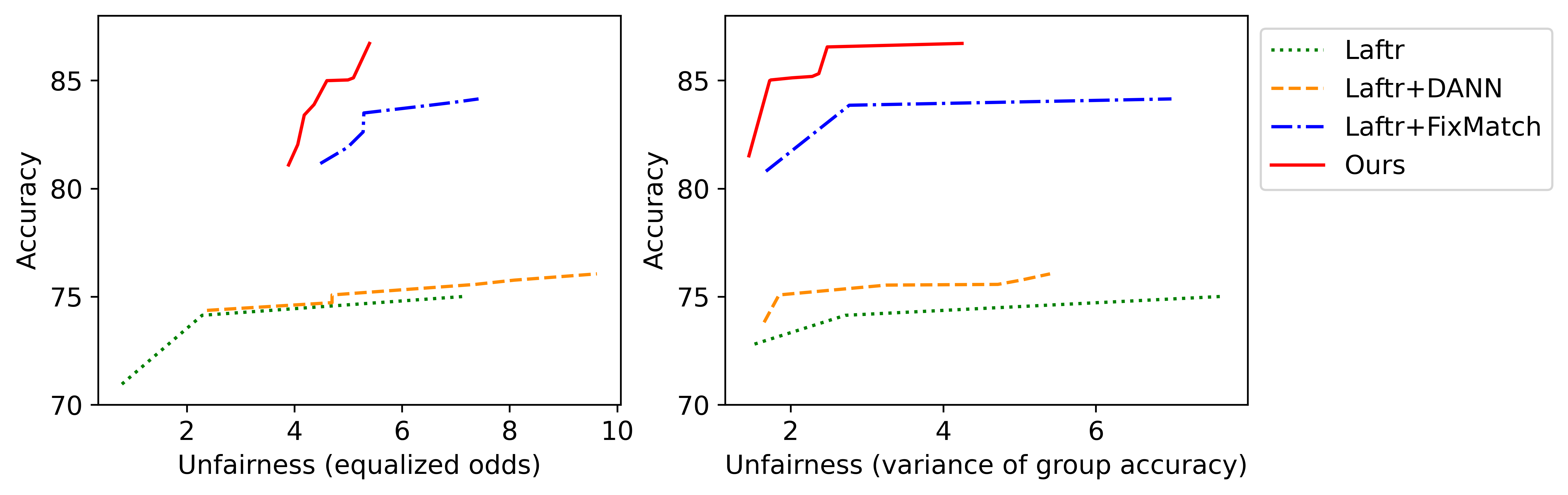}
         \vspace{-0.5em}
         \caption{Pareto frontiers of ours (w/ Laftr) and baselines in the target domain.}
    \end{subfigure}
    \\
    \begin{subfigure}[b]{0.7\textwidth}
        \centering
         \includegraphics[width=\textwidth]{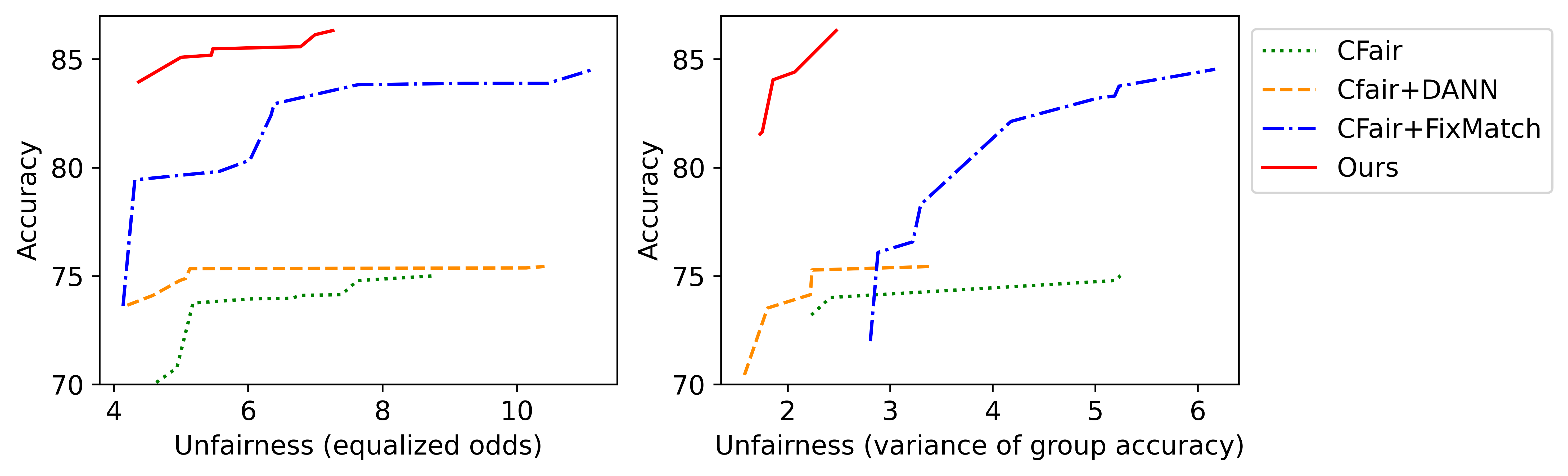}
         \vspace{-0.5em}
         \caption{Pareto frontiers of ours (w/ Cfair) and baselines in the target domain.}
    \end{subfigure}
    \vspace{-0.5em}
    \caption{Comparison of Pareto frontiers. Upper left is preferred. Our method outperforms baseline methods in achieving accuracy and fairness at the same time.}
    \label{fig:pare_ftont}
\end{figure}

\vspace{-0.5em}
\subsection{Evaluation on Real Datasets}
\textbf{Evaluation on images.}
We use UTKFace \cite{zhifei2017cvpr} as the source data and FairFace \cite{karkkainen2019fairface} as the target data. 
Although both are facial images, there is a distribution shift between them due to different image sources. 
We consider a gender classification task with race as the sensitive attribute. 
We use VGG16 \cite{simonyan2014very} as the model and RandAugment \cite{cubuk2020randaugment} (excluding transformations that may change the group) as the transformation function. 
Additional to previous baselines, we also use CFair \cite{conditionalzhao} as the method for in-distribution fairness. 
As shown in Table~\ref{tab:face-vgg_rand}, there is indeed a distribution shift as the source model trained with Laftr or CFair is no longer accurate or fair in the target domain. 
The domain adaptation method has a limited effect on transferring accuracy and fairness. 
As expected, self-training (Laftr+Fixmatch and CFair+Fixmatch) significantly improves the accuracy in the target domain, but the unfairness is high. 
With the proposed fair consistency regularization, our method outperforms it remarkably on fairness with a 70\% decrease in the variance of group accuracy and a 30\% decrease in the equalized odds. 
We further sweep the weights of losses and draw Pareto frontiers. As shown in Figure \ref{fig:pare_ftont}, our method significantly outperforms others in achieving accuracy and fairness at the same time.

\begin{figure}
    \captionsetup{font=footnotesize}
     \centering
     \begin{subfigure}[b]{0.2\textwidth}
         \centering
         \includegraphics[width=\textwidth]{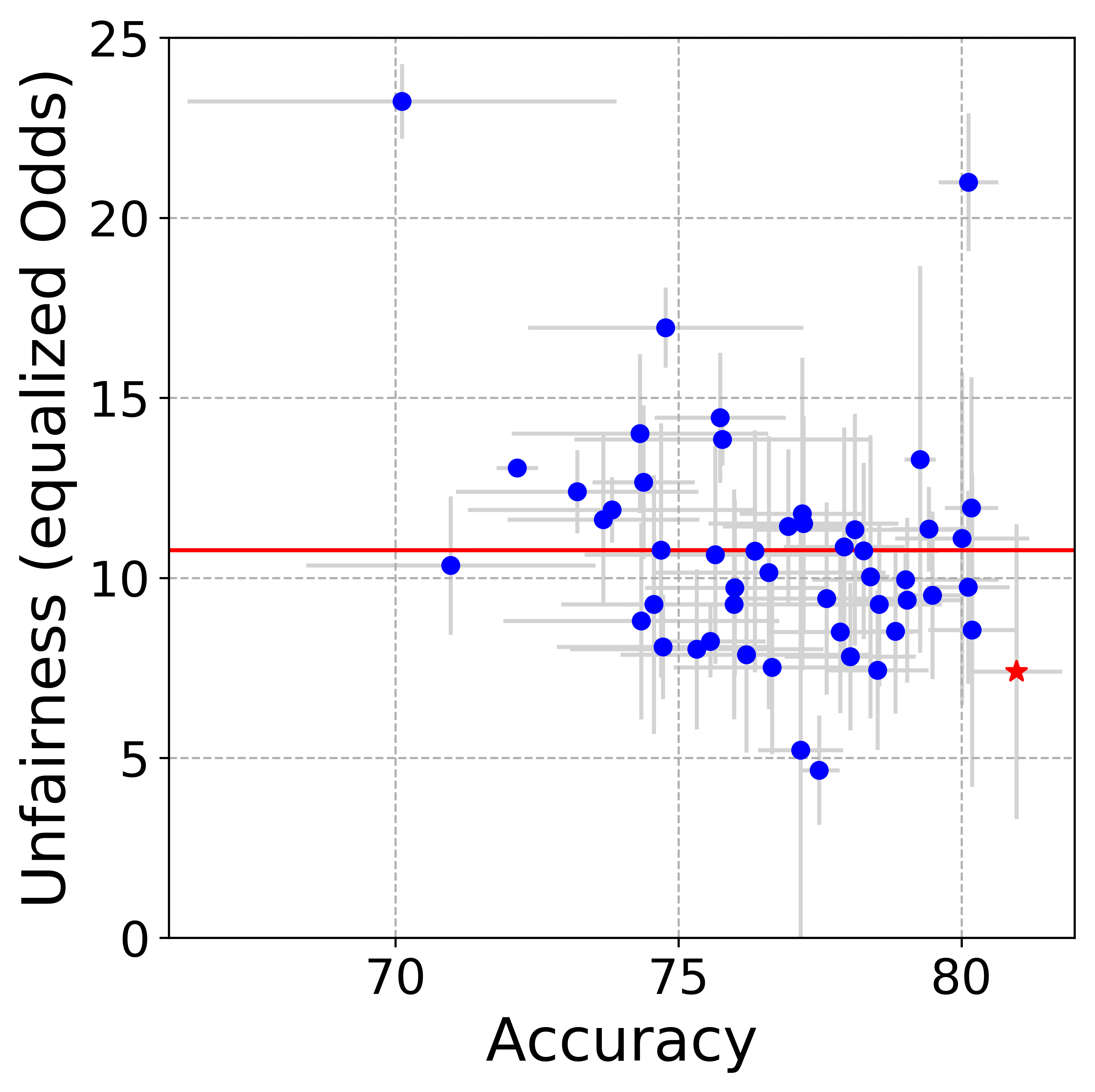}
         \caption{Laftr}\label{fig:adult_laftr}
     \end{subfigure}
     \hfill
     \begin{subfigure}[b]{0.2\textwidth}
         \centering
         \includegraphics[width=\textwidth]{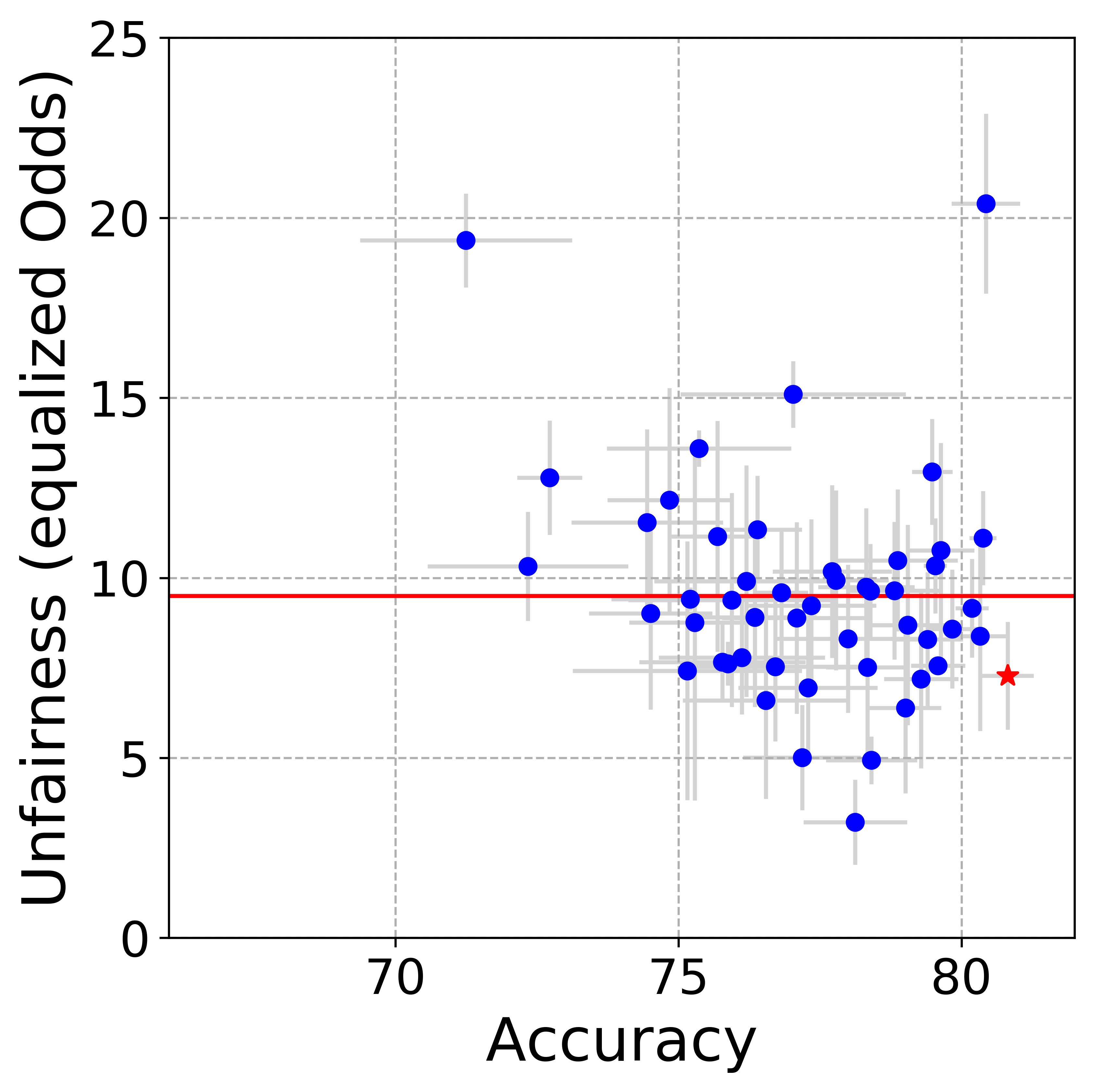}
         \caption{Ours}\label{fig:adult_our}
     \end{subfigure}
     \hfill
     \begin{subfigure}[b]{0.25\textwidth}
         \centering
         \includegraphics[width=\textwidth]{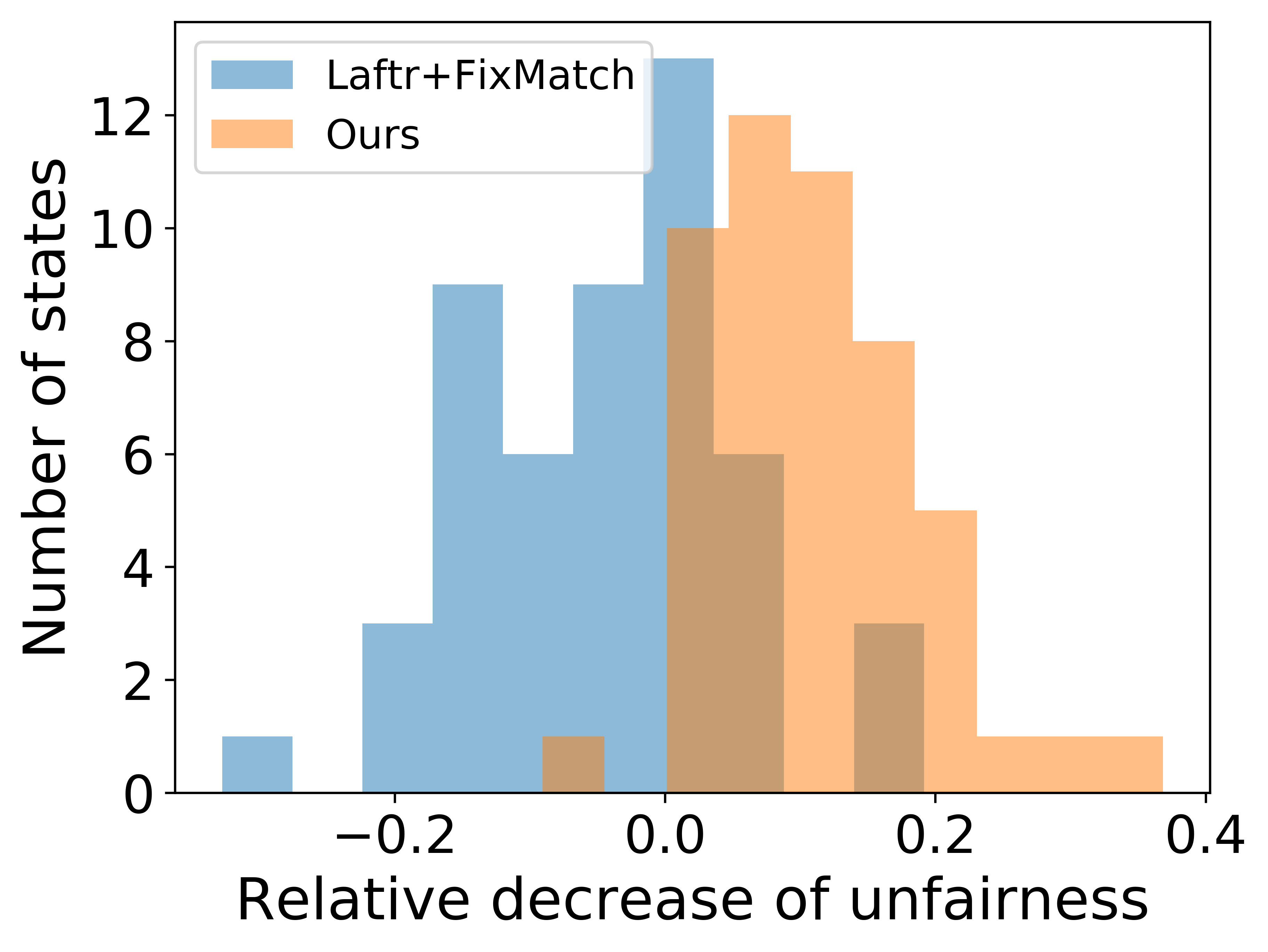}
         \caption{Decrease of unfairness}\label{fig:adult_comp}
     \end{subfigure}
     \caption{Unfairness and accuracy tested on NewAdult. CA as the source domain (red star) and other states as the target domain (blue dots). Red lines indicate the average of unfairness. The relative decrease is calculated by comparing with Laftr.}
     \label{fig:newadult}
\end{figure}

\textbf{Evaluation on tabular data.} 
We further evaluate our method on the NewAdult dataset \cite{ding2021retiring} which contains census data from all states of the United States. We consider gender as the sensitive attribute and do income classification with an MLP as the model. We set CA as the source domain and all the other states as the target domain. We use random perturbation on tabular data (see details in Appendix~\ref{app:exp}) as the transformations. 
Results are shown in Figure~\ref{fig:newadult}. 
When applied to other states, the fair model trained on CA becomes unfair (Figure~\ref{fig:adult_laftr}).
Our method improves the fairness in most states with a slight improvement in accuracy (Figure~\ref{fig:adult_our}). 
Compared with the one without fair consistency regularization, our method achieves better fairness with a decrease in unfairness in most states (Figure~\ref{fig:adult_comp}).

\vspace{-0.5em}
\subsection{Ablation Study}
\vspace{-0.5em}
\textbf{The role of transformation.}
We design transformation functions based on our domain knowledge of latent factors.
To investigate the importance of transformations, we test a weaker set of transformations, which includes only cropping and flipping, on the UTKFace-FairFace experiment and report the performance in Table~\ref{tab:face-vgg_weak}.
Compared with RandAugment in Table~\ref{tab:face-vgg_rand}, consistency under weak transformations leads to a less effective transfer of accuracy since the neighbor generated transformations is much smaller. 
The limited transformations also restrict the performance of our method on tabular data (see Appendix~\ref{app:add-exp}).
Though the ability to transfer accuracy is limited by weak transformations, our method can still make the transfer process fair as there's a significant decrease in unfairness, as shown in Table~\ref{tab:face-vgg_weak}. 

\begin{wrapfigure}{r}{0.4\textwidth}
\captionsetup{font=footnotesize}
\vspace{-1em}
    \centering
     \includegraphics[width=\textwidth]{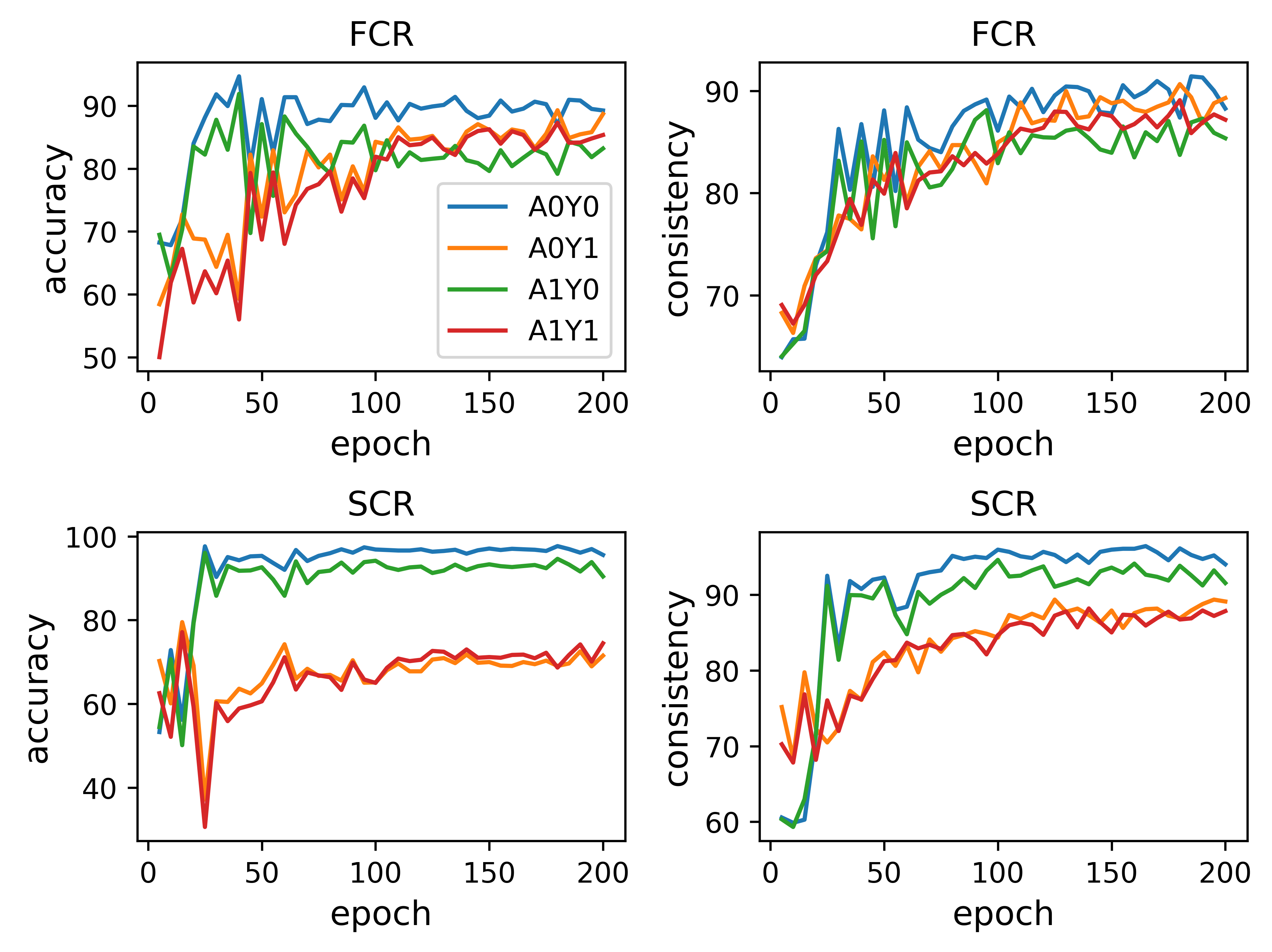}
     \vspace{-2em}
        \caption{Per-group accuracy and consistency. Compared with the standard consistency regularization (SCR), the model trained with fair consistency regularization (FCR) has more balanced consistency and accuracy.}\label{fig:fcr}
\end{wrapfigure}
\textbf{Fair consistency is essential in transferring fairness.} 
To see whether enhanced consistency improves accuracy and whether unbalanced consistency leads to unfairness as suggested by Theorem \ref{th:main}, we evaluate the accuracy and consistency of each group in the UTKFace-FairFace experiment on the target data.
The consistency is measured by testing the model's agreement on the outputs under two random transformations. 
As shown in Figure~\ref{fig:fcr}, groups that obtain higher consistency have higher accuracy, which validates the ability of consistency regularization for transferring accuracy. 
The training methods that use standard consistency regularization (e.g. Laftr+FixMatch) have been observed to be unfair in the target domain. 
Figure~\ref{fig:fcr} shows that it is because the model has imbalanced consistency across groups. 
With our fair consistency regularization, the model gains similar consistency for all groups, resulting in similar group accuracy.

\begin{table}
\captionsetup{font=footnotesize}
\small
\centering
\resizebox{0.8\textwidth}{!}{\begin{tabular}{llllllll}
\toprule
               & \multicolumn{3}{c}{Source}            && \multicolumn{3}{c}{Target}            \\ 
               \cline{2-4} \cline{6-8} 
                              & Acc  & \multicolumn{2}{c}{Unfairness} && Acc & \multicolumn{2}{c}{Unfairness}\\
               \cline{3-4} \cline{7-8}
Method         &    & $V_{acc}$ & $\Delta_{odds}$ &  & & $V_{acc}$ & $\Delta_{odds}$ \\
\midrule
Laftr+FixMatch  & \facc{94.08}{0.70} &  \facc{1.64}{0.46}  & \facc{3.51}{1.46} && \facc{77.05}{0.26} & \facc{12.23}{3.83} & \facc{6.55}{1.54}   \\
CFair+FixMatch  & \facc{94.09}{0.33} &  \facc{0.97}{0.36}  & \facc{2.16}{0.97} && \facc{77.25}{0.21} & \facc{12.93}{2.66} & \facc{9.77}{0.95}   \\

Ours (w/ Laftr)  & \facc{94.25}{0.22} &  \facc{1.06}{0.46}  & \facc{2.09}{0.55} && \facc{77.32}{0.21} & \facc{2.35}{1.67} & \facc{4.27}{1.41}   \\
Ours (w/ CFair)  & \facc{94.24}{0.26} &  \facc{1.67}{0.38}  & \facc{4.43}{0.63} && \facc{77.96}{0.38} & \facc{3.34}{1.08} & \facc{5.70}{1.14}   \\
\bottomrule
\end{tabular}}
\caption{Transfer fairness and accuracy from UTKFace to FairFace with weak transformations}\label{tab:face-vgg_weak}
\end{table}

\begin{wraptable}{r}{0.45\textwidth}
\captionsetup{font=footnotesize}
\small
\centering
\vspace{-1em}
\resizebox{\textwidth}{!}{\begin{tabular}{llll}
\toprule
                      &        Acc & \multicolumn{2}{c}{Unfairness}\\
              \cline{3-4}
Method         &   & $V_{acc}$  & $\Delta_{odds}$ \\
\midrule
Ours  & \facc{85.52}{0.40} & \facc{2.82}{0.87} & \facc{5.70}{0.52}   \\
w/o consistency in target & \facc{82.43}{1.05} &  \facc{6.80}{1.30}  & \facc{5.85}{0.40}  \\
w/o consistency in source & \facc{82.5}{1.58} &  \facc{6.63}{0.71}  & \facc{8.18}{1.27}  \\
w/o dynamic weights & \facc{84.34}{0.19} &  \facc{6.86}{0.50}  & \facc{7.68}{0.81}  \\
w/o updating $g_{tc}$ & \facc{79.13}{0.52} &  \facc{3.49}{0.63}  & \facc{6.65}{1.31}  \\
\bottomrule
\end{tabular}}
\caption{Ablation study on UTKFace-FairFace task} \label{tab:ablation}
\vspace{-0.5em}
\end{wraptable}
\textbf{The role of components in fair consistency regularization.}
Table~\ref{tab:ablation} shows the ablation study. We can see that the consistency in both domains matters. 
Giving every group the same weight instead of using dynamic weights leads to increased unfairness.
Fixing the teacher classifier to be the fair source model, we observe a significant decrease in the accuracy, suggesting the important role of iterative self-training in our algorithm.

\section{Conclusion}
In this paper, we explore how to transfer fairness under distribution shifts. We derive a sufficient condition and present a theory-guided self-training algorithm based on an intra-group expansion assumption. The key component of our algorithm is fair consistency regularization. We simulate different types of distribution shifts with a synthetic dataset and examine our theoretical findings with it. Abundant experiments with synthetic data and real data have shown that our method has superior performance in transferring fairness and accuracy. Like other self-training methods, one limitation of our method is the reliance on a well-defined data transformation set. Future work will relax this limitation for application to more real-world problems.

\section *{Acknowledgements}
This work is supported by National Science Foundation NSF-IIS-FAI program, DOD-ONR-Office of Naval Research, DOD-DARPA-Defense Advanced Research Projects Agency Guaranteeing AI Robustness against Deception (GARD), and Adobe, Capital One and JP Morgan faculty fellowships.

\bibliography{neurips_2022}
\bibliographystyle{plain}

\section*{Checklist}


\begin{enumerate}

\item For all authors...
\begin{enumerate}
  \item Do the main claims made in the abstract and introduction accurately reflect the paper's contributions and scope?
    \answerYes{}
  \item Did you describe the limitations of your work?
    \answerYes{See Section~\ref{app:limit}}
  \item Did you discuss any potential negative societal impacts of your work?
    \answerNA{}
  \item Have you read the ethics review guidelines and ensured that your paper conforms to them?
    \answerYes{}
\end{enumerate}

\item If you are including theoretical results...
\begin{enumerate}
  \item Did you state the full set of assumptions of all theoretical results?
    \answerYes{}
        \item Did you include complete proofs of all theoretical results?
    \answerYes{}
\end{enumerate}

\item If you ran experiments...
\begin{enumerate}
  \item Did you include the code, data, and instructions needed to reproduce the main experimental results (either in the supplemental material or as a URL)?
    \answerYes{}
  \item Did you specify all the training details (e.g., data splits, hyperparameters, how they were chosen)?
    \answerYes{see Section~\ref{app:exp}}
        \item Did you report error bars (e.g., with respect to the random seed after running experiments multiple times)?
    \answerYes{}
        \item Did you include the total amount of compute and the type of resources used (e.g., type of GPUs, internal cluster, or cloud provider)?
    \answerYes{see Section~\ref{app:exp}}
\end{enumerate}

\item If you are using existing assets (e.g., code, data, models) or curating/releasing new assets...
\begin{enumerate}
  \item If your work uses existing assets, did you cite the creators?
    \answerYes{}
  \item Did you mention the license of the assets?
    \answerNA{}
  \item Did you include any new assets either in the supplemental material or as a URL?
    \answerNA{}
  \item Did you discuss whether and how consent was obtained from people whose data you're using/curating?
    \answerNA{}
  \item Did you discuss whether the data you are using/curating contains personally identifiable information or offensive content?
    \answerNA{}
\end{enumerate}

\item If you used crowdsourcing or conducted research with human subjects...
\begin{enumerate}
  \item Did you include the full text of instructions given to participants and screenshots, if applicable?
    \answerNA{}
  \item Did you describe any potential participant risks, with links to Institutional Review Board (IRB) approvals, if applicable?
    \answerNA{}
  \item Did you include the estimated hourly wage paid to participants and the total amount spent on participant compensation?
    \answerNA{}
\end{enumerate}

\end{enumerate}


\newpage
\appendix
{\centering \bf \LARGE
    Supplementary Material
}
\section{More on Related Work}\label{app:sec:related}

\textbf{Fair machine learning.}
Generally, fair machine learning methods fall into three categories:
pre-processing, in-processing, and post-processing  \cite{survey, caton2020fairness}. In this paper, we focus on in-processing methods that modify learning algorithms to remove discrimination during the training process. As for fair classification, several approaches have been proposed including fair representation learning \cite{pmlr-v28-zemel13, fair_vae, beutel2017data, zhang2018mitigating, david_learning_fair, control2019, disentangle2019, conditionalzhao}, fairness-constrained optimization \cite{Donini_erm_fair, reduction}, causal methods \cite{counter2017, causal2019, fairness2020},  and many other approaches with different techniques \cite{fairmeta, fairmixup, goel2021model}. All of those works are for in-distribution fairness, and we investigate out-of-distribution fairness in this paper. We use LAFTR \cite{david_learning_fair}, an adversarial learning method that shows advanced performance on fairness \cite{reddy2021benchmarking}, to learn a fair model in the source domain and adapt it to the target domain. We also test CFair\cite{zhao2019conditional} in our experiments. Many metrics of fairness have been proposed \cite{corbettdavies2018measure} including demographic parity \cite{5360534}, equalized opportunity, and equalized odds \cite{hardt2016equality} which are most widely adopted. In this paper, we use equalized odds to measure unfairness in both domains.

\textbf{Distribution shifts.} 
In many real-world applications, distribution shifts are unavoidable. The goal of existing work addressing distribution shifts is simply to transfer accuracy. \cite{koh2021wilds} propose a benchmark of in-the-wild datasets to study the real distribution shifts. We follow their category of distribution shifts, including subpopulation shifts and domain shifts. Their empirical results on many state-of-the-art methods show that self-training outperforms others on image datasets significantly while having limited performance due to the limited data augmentation on non-image modalities \cite{Sagawa2021ExtendingTW}. This finding aligns with our experimental results. \cite{wiles2022a} conduct a fine-grained analysis of various distribution shifts based on an underlying data generation assumption similar to ours. They also use 3dshapes dataset to simulate different types of distribution shifts. Additional to accuracy, we aim to transfer fairness at the same time in this paper.

\textbf{Domain adaptation and self-training.}
Inspired by the theoretical work \cite{ben2010}, numerous distribution matching approaches have emerged for domain adaption over the past decade. Domain-adversarial training \cite{ganin2016domain} and many of its variants \cite{tzeng2017adversarial, long2018conditional, pmlr-v80-hoffman18a, tachet2020domain} that aim at matching the distribution of two domains in the feature space have shown encouraging results in many applications. However, recent studies \cite{wu2019domain, zhao2019learning, li2020rethinking} show that such methods may fail in many cases since they only optimize part of the theoretical bound. We test DANN \cite{ganin2016domain}, and MMD \cite{long2015learning}, two distribution matching methods in our experiments, and also find them less effective in transferring accuracy and fairness. Recently, another line of work that uses self-training draws increasing attention \cite{zhang2021semisupervised, berthelot2021adamatch}. Those methods enjoy guarantees \cite{colin2021, tianle2021} and demonstrate superior empirical results with desirable properties such as robustness to spurious features \cite{kumar2020understanding, chen2020self, liu2021cycle} and robustness to dataset imbalance \cite{liu2021selfsupervised}. However, all of those work on domain adaptation only aims at transferring accuracy. Although there is work that studies fairness issues in current domain adaptation methods \cite{lan2017discriminatory} and proposes to alleviate it by balancing the data \cite{jing2021towards, wang2021towards, zhu2022the}, fair domain adaptation is still under-explored. Based on the findings that the model's consistency to input transformations is important to generalization \cite{zhu2021understanding} and is a core component of self-training \cite{shu2018dirt, sohn2020fixmatch, grill2020bootstrap}, we improve the consistency regularization in \cite{sohn2020fixmatch} to achieve fair transferring.

\section{Proof and More Discussion of Fairness under Distribution Shifts}\label{app:shift}
\begin{lemma}\label{lemma:con}
Under Assumption~\ref{asum:under}, for a subpopulation shift that is caused by the shift of the marginal distribution of factor $Y^i$, we have
$\Pbb_S(X|Y^i=y^i)=\Pbb_T(X|Y^i=y^i), \forall y^i\in \Ycal^i$.
\end{lemma}
\begin{proof}
Under Assumption~\ref{asum:under}, $\Pbb_S(X|Y^{1:K}=\by^{1:K})=\Pbb_T(X|Y^{1:K}=\by^{1:K})$. 
Since the shift only happens on the factor $Y^i$, the marginal distribution of other factors remains the same in the two domains,
$\Pbb_S(Y^{\{1:K\}\setminus i})=\Pbb_T(Y^{\{1:K\}\setminus i})$ where we use $\{1:K\}\setminus i$ to denote $1,..,i-1,i+1,...,K$. Then
\begin{align*}
    \Pbb_S(X|Y^i=y^i)&=\sum_{\by^{\{1:K\}\setminus i}}\Pbb_S(X, Y^{\{1:K\}\setminus i}=\by^{\{1:K\}\setminus i}|Y^i=y^i) \\
    &=\sum_{\by^{\{1:K\}\setminus i}}\Pbb_S(Y^{\{1:K\}\setminus i}=\by^{\{1:K\}\setminus i})\Pbb_S(X|Y^i=y^i,Y^{\{1:K\}\setminus i}=\by^{\{1:K\}\setminus i}) \\
    &=\sum_{\by^{\{1:K\}\setminus i}}\Pbb_T(Y^{\{1:K\}\setminus i}=\by^{\{1:K\}\setminus i})\Pbb_T(X|Y^i=y^i,Y^{\{1:K\}\setminus i}=\by^{\{1:K\}\setminus i}) \\
    &=\Pbb_T(X|Y^i=y^i)
\end{align*}
where the second line holds because of the independence of the latent factors $Y^1, ..., Y^K$.
\end{proof}

Now, we restate Proposition~\ref{prop:subshift} and provide the proof.
\begin{proposition}[Transfer of fairness under subpopulation shift]
Consider the subpopulation shift that is caused by the shifted marginal distribution of a nuisance factor $Y^i$ (i.e., $\Pbb_S(Y^i)\neq \Pbb_T(Y^i)$), while $\Ycal_S^i=\Ycal_T^i=\Ycal^i$. If model $f$ is strictly fair in the source domain under any value of factor $Y^i$ satisfying
$\Pbb_S(f(X)=y^l | Y^a=0, Y^l=y^l, Y^i=y^i)=\Pbb_S(f(X)=y^l | Y^a=1, Y^l=y^l, Y^i=y^i), \forall y^i\in \Ycal^i, y^l\in\{0,1\}$, 
then $f$ is also fair in target domain with $\Delta_{odds}=0$.
\end{proposition}
\begin{proof}
In the target domain, the equalized odds (unfairness) is defined as
\begin{align*}
    \Delta_{odds}=\frac{1}{2}\sum_{y^l=0}^1\big|\Pbb_T(f(X)=y^l | Y^a=0, Y^l=y^l) - \Pbb_T(f(X)=y^l | Y^a=1, Y^l=y^l)\big|.
\end{align*}
Since all latent factors are independent, we have
\begin{align*}
    \Pbb_T\left(f(X)=y^l | Y^a=0, Y^l=y^l\right) &= \sum_{y^i\in \Ycal^i}\Pbb_T(Y^i=y^i)\Pbb_T\left(f(X)=y^l | Y^a=0, Y^l=y^l, Y^i=y^i\right) 
\end{align*}
and 
\begin{align*}
    \Pbb_T\left(f(X)=y^l | Y^a=1, Y^l=y^l\right) &= \sum_{y^i\in \Ycal^i}\Pbb_T(Y^i=y^i)\Pbb_T\left(f(X)=y^l | Y^a=1, Y^l=y^l, Y^i=y^i\right) .
\end{align*}
Therefore, the $\Delta_{odds}$ in the target domain can be decomposed into 
\begin{align*}
    \Delta_{odds}=\frac{1}{2}\sum_{y^l=0}^1\big|\sum_{y^i\in \Ycal^i}\Pbb_T(Y^i=y^i)&\big(\Pbb_T\left(f(X)=y^l | Y^a=0, Y^l=y^l, Y^i=y^i\right) \\
    &-\Pbb_T\left(f(X)=y^l | Y^a=1, Y^l=y^l, Y^i=y^i\right)\big)\big|.
\end{align*}
Since two domains share the same underlying data generative model, and the distribution shift is caused by the shift of the marginal distribution of factor $Y^i$, from Lemma~\ref{lemma:con}, we have
\begin{align*}
   \Pbb_T\left(X | Y^a=0, Y^l=y^l, Y^i=y^i\right)=\Pbb_S\left(X | Y^a=0, Y^l=y^l, Y^i=y^i\right)
\end{align*}
and 
\begin{align*}
   \Pbb_T\left(X | Y^a=1, Y^l=y^l, Y^i=y^i\right)=\Pbb_S\left(X | Y^a=1, Y^l=y^l, Y^i=y^i\right).
\end{align*}
Thus the conditional distribution of the model's prediction also remains, as
\begin{align*}
   \Pbb_T\left(f(X) | Y^a=0, Y^l=y^l, Y^i=y^i\right)=\Pbb_S\left(f(X) | Y^a=0, Y^l=y^l, Y^i=y^i\right)
\end{align*}
and 
\begin{align*}
   \Pbb_T\left(f(X) | Y^a=1, Y^l=y^l, Y^i=y^i\right)=\Pbb_S\left(f(X) | Y^a=1, Y^l=y^l, Y^i=y^i\right).
\end{align*}
In this case, if the source model is strictly fair that $\forall y^i\in \Ycal^i, y^l\in\{0,1\}$ the following holds
\begin{align*}
    \Pbb_S(f(X)=y^l | Y^a=0, Y^l=y^l, Y^i=y^i)=\Pbb_S(f(X)=y^l | Y^a=1, Y^l=y^l, Y^i=y^i), 
\end{align*}
then it is also fair in the target domain with
\begin{align*}
    \Delta_{odds}=\frac{1}{2}\sum_{y^l=0}^1\big|\sum_{y^i\in \Ycal^i}\Pbb_T(Y^i=y^i)&\big(\Pbb_S\left(f(X)=y^l | Y^a=0, Y^l=y^l, Y^i=y^i\right) \\
    &-\Pbb_S\left(f(X)=y^l | Y^a=1, Y^l=y^l, Y^i=y^i\right)\big)\big|=0.
\end{align*}
\end{proof}

This proposition explains why the fair source model is also fair in the target domain under Sshift 1 in our experiments (see Section~\ref{exp:syn}). In addition to shifts of nuisance factors, the subpopulation shifts can also be caused by the marginal distribution shift of the label and sensitive attribute. The following proposition argues that the fair model is also in the target domain under such distribution shifts.

\begin{proposition}[Transfer of fairness under subpopulation shift of sensitive attribute]
Consider the subpopulation shift that is caused by the shifted marginal distribution of sensitive attribute $Y^a$ (i.e., $\Pbb_S(Y^a)\neq \Pbb_T(Y^a)$), while $\Ycal_S^a=\Ycal_T^a=\Ycal^a=\{0,1\}$. If model $f$ is fair in the source domain with $\Delta^S_{odds}=0$, then it is also fair in the target domain with $\Delta^T_{odds}=0$.
\end{proposition}
\begin{proof}
The proof is similar to the previous one. Since 
\begin{align*}
    \Delta^S_{odds}=\frac{1}{2}\sum_{y^l=0}^1\big|\Pbb_S(f(X)=y^l | Y^a=0, Y^l=y^l) - \Pbb_S(f(X)=y^l | Y^a=1, Y^l=y^l)\big|,
\end{align*}
and from Lemma~\ref{lemma:con} we know that
\begin{align*}
    \Pbb_S(X | Y^a=0, Y^l=y^l)&=\Pbb_T(X | Y^a=0, Y^l=y^l) \\
    \Pbb_S(X | Y^a=1, Y^l=y^l)&=\Pbb_T(X | Y^a=1, Y^l=y^l), 
\end{align*}
thus,
\begin{align*}
    \Delta^T_{odds}&=\frac{1}{2}\sum_{y^l=0}^1\big|\Pbb_T(f(X)=y^l | Y^a=0, Y^l=y^l) - \Pbb_T(f(X)=y^l | Y^a=1, Y^l=y^l)\big| \\
    &=\frac{1}{2}\sum_{y^l=0}^1\big|\Pbb_S(f(X)=y^l | Y^a=0, Y^l=y^l) - \Pbb_S(f(X)=y^l | Y^a=1, Y^l=y^l)\big|\\
    &=\Delta^S_{odds}=0 .
\end{align*}
\end{proof}
Such a result also holds for subpopulation shifts caused by the shift of label $Y^l$. This proposition explains why the fair source model is also fair in the target domain under Sshift 2 in our experiments (see Section~\ref{exp:syn}). It suggests that encouraging fairness is able to alleviate spurious correlation. We leave more studies on this interesting finding to future work.

\textbf{Remark.} All the above analyses are based on the population distribution where ${\Pbb_S(f(X)=y^l|Y^a=0, Y^l=y^l)=\Ebb_{\Pbb_S(X,Y^{1:K})}(f(X)=y^l|Y^a=0, Y^l=y^l)}$. In practice, it is estimated by finite samples. Insufficient samples would cause estimation errors in fairness and bring another challenge for transferring fairness. In this paper, we only consider the fairness measured by population distribution. Future work will investigate the impact of estimation error on transferring fairness and the way to resolve it.

\section{Proof of the Sufficient Condition for Transferring Fairness} \label{app:proof}
Our proof is based on the theory in \cite{colin2021}.
\begin{theorem}
(Restatement of Lemma A.8 in \cite{colin2021}) We assume that $U_a^y$ satisfies $(\bar{\alpha}, \bar{c})$-multiplicative expansion for $\varepsilon_{U_a^y}(g_{tc})\leq\bar{\alpha}<1/3$ and $\bar{c}>3$. We define $c\triangleq\min\{1/\bar{\alpha}, \bar{c}\}$. Then for any classifier $g: \Xcal \rightarrow \Ycal$, the error of it on the group $U_a^y$ is upper bounded as:
\begin{align*}
    \varepsilon_{U_a^y}(g) \leq \frac{c+1}{c-1} L_{U_a^y}(g, g_{tc}) + \frac{2c}{c-1}R_{U_a^y}(g)- \varepsilon_{U_a^y}(g_{tc}) \label{eq:upper}
\end{align*}
\end{theorem}

\begin{theorem}\label{thm:upper}
(A restricted version of the above theorem) We assume that $U_a^y$ satisfies $(\bar{\alpha}, \bar{c})$-multiplicative expansion for $\varepsilon_{U_a^y}(g_{tc})\leq\bar{\alpha}<1/3$ and $\bar{c}>3$. We define $c\triangleq\min\{1/\bar{\alpha}, \bar{c}\}$. Then for any classifier $g: \Xcal \rightarrow \Ycal$ satisfies $L_{U_a^y}(g, g_{tc})\leq \varepsilon_{U_a^y}(g_{tc})$, the error of it on the group $U_a^y$ is upper bounded as:
\begin{align*}
    \varepsilon_{U_a^y}(g) \leq \frac{2}{c-1} \varepsilon_{U_a^y}(g_{tc}) + \frac{2c}{c-1}R_{U_a^y}(g) 
\end{align*}
\end{theorem}
\begin{proof}
\begin{align*}
    \varepsilon_{U_a^y}(g) &\leq \frac{c+1}{c-1} L_{U_a^y}(g, g_{tc}) + \frac{2c}{c-1}R_{U_a^y}(g)- \varepsilon_{U_a^y}(g_{tc}) \\
    \varepsilon_{U_a^y}(g) &\leq \frac{2}{c-1} \varepsilon_{U_a^y}(g_{tc}) + \frac{2c}{c-1}R_{U_a^y}(g) \tag*{(because $L_{U_a^y}(g, g_{tc})\leq \varepsilon_{U_a^y}(g_{tc})$)}
\end{align*}
\end{proof}

\begin{theorem} \label{thm:lower}
If $L_{U_a^y}(g, g_{tc})\leq \varepsilon_{U_a^y}(g_{tc})$, we have
\begin{align*}
    \varepsilon_{U_a^y}(g) \geq \varepsilon_{U_a^y}(g_{tc}) - L_{U_a^y}(g, g_{tc}) 
\end{align*}
\end{theorem}
\begin{proof}
By triangle inequality. 
\end{proof}
Now, we restate Theorem~\ref{th:main} and provide the proof.
\begin{theorem} 
Suppose we have a teacher classifier $g_{tc}$ with bounded unfairness such that $|\varepsilon_{U_a^y}(g_{tc})-\varepsilon_{U_{a'}^{y'}}(g_{tc})|\leq \gamma, \forall a,a'\in \Acal$ and $y,y'\in \Ycal$.
We assume intra-group expansion where $U_a^y$ satisfies $(\bar{\alpha}, \bar{c})$-multiplicative expansion and $\varepsilon_{U_a^y}(g_{tc})\leq\bar{\alpha}<1/3$ and $\bar{c}>3, \forall a,y$. We define $c\triangleq\min\{1/\bar{\alpha}, \bar{c}\}$, and set $\mu \leq \varepsilon_{U_a^y}(g_{tc}), \forall a,y$.
If we train our classifier with the algorithm
\begin{align}
    &\min_{g\in G}\max_{a, y} R_{U_a^y}(g) \\
    \textnormal{s.t. } &\quad L_{U_a^y}(g, g_{tc})\leq \mu \quad \forall a,y \notag
\end{align}
then the error and unfairness of the optimal solution $\hat{g}$ on the distribution $U$ are bounded with
\begin{align*}
    \varepsilon(\hat{g}) &\leq \frac{2}{c-1} \varepsilon_U(g_{tc}) + \frac{2c}{c-1}R_U(\hat{g}), \\
    \Delta_{odds}(\hat{g}) &\leq \frac{2}{c-1}(\gamma + \mu + c\max_{a, y} R_{U_a^y}(\hat{g})).
\end{align*}
\end{theorem}
\begin{proof}
The upper bound of error is derived from Theorem~\ref{thm:upper}. For the unfairness, by definition 
\begin{align*}
    \Delta_{odds}(\hat{g})=\frac{1}{2}\left(\left|\varepsilon_{U_0^0}(\hat{g}) - \varepsilon_{U_1^0}(\hat{g})\right|+
    \left|\varepsilon_{U_0^1}(\hat{g}) - \varepsilon_{U_1^1}(\hat{g})\right|\right).
\end{align*}
Based on the upper bound of group error from Theorem~\ref{thm:upper}, and the lower bound of it from Theorem~\ref{thm:lower}, we have
\begin{align*}
    &\left|\varepsilon_{U_0^0}(\hat{g}) - \varepsilon_{U_1^0}(\hat{g})\right| \\
    &\leq \max\Big\{\frac{2}{c-1} \gamma + \frac{2}{c-1}L_{U_1^0}(\hat{g}, g_{tc}) + \frac{2c}{c-1}R_{U_0^0}(\hat{g}),  \\
    &\quad\quad\quad\quad\frac{2}{c-1} \gamma + \frac{2}{c-1}L_{U_0^0}(\hat{g}, g_{tc}) + \frac{2c}{c-1}R_{U_1^0}(\hat{g})\Big\} \tag*{(because $c>3$)} \\
    &= \frac{2}{c-1} \gamma + \frac{2}{c-1}\max\left\{L_{U_1^0}(\hat{g}, g_{tc}) + cR_{U_0^0}(\hat{g}), L_{U_0^0}(\hat{g}, g_{tc}) + cR_{U_1^0}(\hat{g})\right\} \\
    &\leq \frac{2}{c-1}(\gamma + \mu+c\max_{a} \ R_{U_a^0}(\hat{g})). \\
\end{align*}
Therefore, 
\begin{align*}
    \Delta_{odds}(\hat{g}) &\leq  \frac{2}{c-1}\left(\gamma + \mu+\frac{c}{2}\left(\max_{a} \ R_{U_a^0}(\hat{g})+\max_{a} \ R_{U_a^1}(\hat{g})\right)\right) \\
    &\leq  \frac{2}{c-1}(\gamma + \mu+c\max_{a,y} \ R_{U_a^y}(\hat{g})).
\end{align*}
\end{proof}
\textbf{Upper bound of $V_{acc}$.}  
From Theorem~\ref{thm:upper} we know that the group accuracy is upper bounded by $\varepsilon_{U_a^y}(\hat{g}) \leq \frac{2}{c-1} \varepsilon_{U_a^y}(g_{tc}) + \frac{2c}{c-1}R_{U_a^y}(\hat{g})$. 
The variance of group accuracy is defined as 
\begin{align*}
    V_{acc}(\hat{g})&=Var(\{\Pbb(\hat{Y}=y|A=a, Y=y), \forall a,y\})\\ 
    &=Var(\{\varepsilon_{U_a^y}(\hat{g}), \forall a,y\})
\end{align*}
If we assume the same estimation error for all the groups when we use the upper bound to estimate the group accuracy with $\varepsilon_{U_a^y}(\hat{g}) = \frac{2}{c-1} \varepsilon_{U_a^y}(g_{tc}) + \frac{2c}{c-1}R_{U_a^y}(\hat{g})$, then 
\begin{align*}
    V_{acc}(\hat{g})&=Var\left(\left\{\frac{2}{c-1} \varepsilon_{U_a^y}(g_{tc}) + \frac{2c}{c-1}R_{U_a^y}(\hat{g}), \forall a,y\right\}\right)\\ 
\end{align*}
When the teacher classifier has bounded unfairness with $|\varepsilon_{U_a^y}(g_{tc})-\varepsilon_{U_{a'}^{y'}}(g_{tc})|\leq \gamma, \forall a,a',y,y'$, the variance of group accuracy would be mainly affected by the variance of group consistency loss $Var(\{R_{U_a^y}(\hat{g}), \forall a,y\})$. Therefore, this upper bound also suggests us to balance the consistency loss while minimizing it.

\textbf{Multi-sensitive attribute and multi-class cases.}
It is obvious that the Theorem~\ref{th:main} still holds for the binary-sensitive attribute and multi-class case where $\Ycal=\{1,2,.., M\}$. 
For the multi-sensitive attribute case, the key problem is how to define the unfairness. 
If we define the equalized odds in general cases to be the following one, then it is easy to see that Theorem~\ref{th:main} still holds.
\begin{align*}
    \Delta_{odds}(\hat{g})=\frac{1}{|\Ycal|}\sum_{y\in\Ycal}\max_{a,a'\in\Acal}\left|\varepsilon_{U_a^y}(\hat{g}) - \varepsilon_{U_{a'}^y}(\hat{g})\right|
\end{align*}


\section{Details of Experiments} \label{app:exp}
\subsection{Synthetic Dataset} \label{app:3dshapes}

The 3dshapes dataset \footnote{https://github.com/deepmind/3d-shapes} \cite{pmlr-v80-kim18b} contains 480000 RGB images (the size is $64 \times 64 \times 3$) of 3D objects.
Every image is generated by six latent factors (\textit{shape}, \textit{object hue}, \textit{scale}, \textit{orientation}, \textit{floor hue}, \textit{wall hue}) which are annotated along with images. The sample spaces of these factors are shown in Table~\ref{tab:3dshape_factor}.
\begin{table}[h]
    \centering
    \begin{tabular}{l|l}
    \toprule
    Factor & Sample space \\
    \midrule
    shape & 4 values in [0, 1, 2, 3] \\
    object hue & 10 values linearly spaced in [0, 1] \\
    scale & 8 values linearly spaced in [0, 1] \\
    orientation & 15 values linearly spaced in [-30, 30]\\
    floor hue & 10 values linearly spaced in [0, 1] \\
    wall hue & 10 values linearly spaced in [0, 1] \\
\bottomrule
    \end{tabular}
    \caption{Latent factors in 3dshapes dataset.}
    \label{tab:3dshape_factor}
\end{table}

\textbf{How to simulate different types of distribution shift?} 
By varying the marginal distribution of latent factors and then sample images according to the distribution of latent factors, we can simulate different distribution shifts. 
In this paper, we set the image as input $X$, and set class $Y=\textit{shape}$, sensitive attribute $A=\textit{object hue}$, and a nuisance factor that might shift to be $D=\textit{scale}$. 
We consider a binary case, and restrict the \textit{shape} to be in $\{0,1\}$ (i.e. $\{\text{cube}, \text{cylinder}\}$) and \textit{object hue} to be in $\{0,1\}$ (i.e. $\{\text{red}, \text{yellow}\}$). 
In our experiments, we simulate four types of distribution shift. Their specific settings are shown in Table~\ref{tab:3dshape_shifts}. 
We show examples from two domains under different shifts in Figure~\ref{fig:3dshape_shifts}.

\begin{figure*}[t!]

\caption{Randomly sampled examples from two domains under different distribution shifts.}
\label{fig:3dshape_shifts}

\rotatebox[origin=c]{90}{Source\strut}
\vspace{-0.5cm}
    \begin{subfigure}{0.09\textwidth}
        \stackinset{c}{}{t}{-.2in}{\textbf{}}{%
            \includegraphics[width=0.9\textwidth]{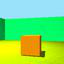}}
    \end{subfigure}%
    \begin{subfigure}{0.09\textwidth}
        \stackinset{c}{}{t}{-.2in}{\textbf{
        }}{%
            \includegraphics[width=0.9\textwidth]{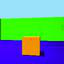}}
    \end{subfigure}%
    \begin{subfigure}{0.09\textwidth}
        \stackinset{c}{}{t}{-.2in}{\textbf{
        }}{%
            \includegraphics[width=0.9\textwidth]{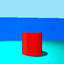}}
        \label{fig:smad7_reproduced_ensemble_ski}
    \end{subfigure}
    \begin{subfigure}{0.09\textwidth}
        \stackinset{c}{}{t}{-.2in}{\textbf{
        }}{%
            \includegraphics[width=0.9\textwidth]{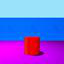}}
        \label{fig:smad7_reproduced_ensemble_ski}
    \end{subfigure}
    \begin{subfigure}{0.09\textwidth}
        \stackinset{c}{}{t}{-.2in}{\textbf{
        }}{%
            \includegraphics[width=0.9\textwidth]{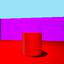}}
        \label{fig:smad7_reproduced_ensemble_ski}
    \end{subfigure}
    \begin{subfigure}{0.09\textwidth}
        \stackinset{c}{}{t}{-.2in}{\textbf{
        }}{%
            \includegraphics[width=0.9\textwidth]{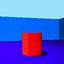}}
        \label{fig:smad7_reproduced_ensemble_ski}
    \end{subfigure}
    \begin{subfigure}{0.09\textwidth}
        \stackinset{c}{}{t}{-.2in}{\textbf{
        }}{%
            \includegraphics[width=0.9\textwidth]{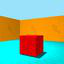}}
        \label{fig:smad7_reproduced_ensemble_ski}
    \end{subfigure}
    \begin{subfigure}{0.09\textwidth}
        \stackinset{c}{}{t}{-.2in}{\textbf{
        }}{%
            \includegraphics[width=0.9\textwidth]{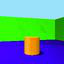}}
        \label{fig:smad7_reproduced_ensemble_ski}
    \end{subfigure}
    \begin{subfigure}{0.09\textwidth}
        \stackinset{c}{}{t}{-.2in}{\textbf{
        }}{%
            \includegraphics[width=0.9\textwidth]{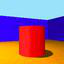}}
        \label{fig:smad7_reproduced_ensemble_ski}
    \end{subfigure}
    \begin{subfigure}{0.09\textwidth}
        \stackinset{c}{}{t}{-.2in}{\textbf{
        }}{%
            \includegraphics[width=0.9\textwidth]{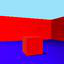}}
    \end{subfigure}
 
\rotatebox[origin=c]{90}{ Target\strut}
    \begin{subfigure}{0.09\textwidth}
        \stackinset{c}{}{t}{-.2in}{\textbf{}}{%
            \includegraphics[width=0.9\textwidth]{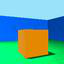}}
        \label{fig:smad7_reproduced_ensemble_Smad7mRNA}
    \end{subfigure}%
    \begin{subfigure}{0.09\textwidth}
        \stackinset{c}{}{t}{-.2in}{\textbf{
        }}{%
            \includegraphics[width=0.9\textwidth]{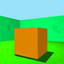}}
        \label{fig:smad7_reproduced_ensemble_Smad7}
    \end{subfigure}%
    \begin{subfigure}{0.09\textwidth}
        \stackinset{c}{}{t}{-.2in}{\textbf{
        }}{%
            \includegraphics[width=0.9\textwidth]{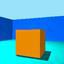}}
        \label{fig:smad7_reproduced_ensemble_ski}
    \end{subfigure}
    \begin{subfigure}{0.09\textwidth}
        \stackinset{c}{}{t}{-.2in}{\textbf{
        }}{%
            \includegraphics[width=0.9\textwidth]{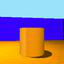}}
        \label{fig:smad7_reproduced_ensemble_ski}
    \end{subfigure}
    \begin{subfigure}{0.09\textwidth}
        \stackinset{c}{}{t}{-.2in}{\textbf{
        }}{%
            \includegraphics[width=0.9\textwidth]{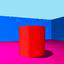}}
        \label{fig:smad7_reproduced_ensemble_ski}
    \end{subfigure}
    \begin{subfigure}{0.09\textwidth}
        \stackinset{c}{}{t}{-.2in}{\textbf{
        }}{%
            \includegraphics[width=0.9\textwidth]{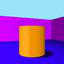}}
        \label{fig:smad7_reproduced_ensemble_ski}
    \end{subfigure}
    \begin{subfigure}{0.09\textwidth}
        \stackinset{c}{}{t}{-.2in}{\textbf{
        }}{%
            \includegraphics[width=0.9\textwidth]{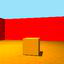}}
        \label{fig:smad7_reproduced_ensemble_ski}
    \end{subfigure}
    \begin{subfigure}{0.09\textwidth}
        \stackinset{c}{}{t}{-.2in}{\textbf{
        }}{%
            \includegraphics[width=0.9\textwidth]{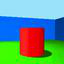}}
        \label{fig:smad7_reproduced_ensemble_ski}
    \end{subfigure}
    \begin{subfigure}{0.09\textwidth}
        \stackinset{c}{}{t}{-.2in}{\textbf{
        }}{%
            \includegraphics[width=0.9\textwidth]{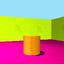}}
        \label{fig:smad7_reproduced_ensemble_ski}
    \end{subfigure}
    \begin{subfigure}{0.09\textwidth}
        \stackinset{c}{}{t}{-.2in}{\textbf{
        }}{%
            \includegraphics[width=0.9\textwidth]{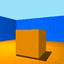}}
    \end{subfigure}
    \subcaption{Sshift 1}
    \vspace{-0.3cm}

\rotatebox[origin=c]{90}{ Source\strut}
\vspace{-0.5cm}
    \begin{subfigure}{0.09\textwidth}
    
            \stackinset{c}{}{t}{-.2in}{\textbf{}}{%
                \includegraphics[width=0.9\textwidth]{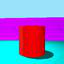}}
        \end{subfigure}%
    \begin{subfigure}{0.09\textwidth}
            \stackinset{c}{}{t}{-.2in}{\textbf{
            }}{%
                \includegraphics[width=0.9\textwidth]{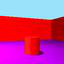}}
            \label{fig:smad7_reproduced_ensemble_Smad7}
        \end{subfigure}%
    \begin{subfigure}{0.09\textwidth}
            \stackinset{c}{}{t}{-.2in}{\textbf{
            }}{%
                \includegraphics[width=0.9\textwidth]{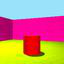}}
            \label{fig:smad7_reproduced_ensemble_ski}
        \end{subfigure}
    \begin{subfigure}{0.09\textwidth}
            \stackinset{c}{}{t}{-.2in}{\textbf{
            }}{%
                \includegraphics[width=0.9\textwidth]{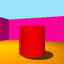}}
            \label{fig:smad7_reproduced_ensemble_ski}
        \end{subfigure}
    \begin{subfigure}{0.09\textwidth}
            \stackinset{c}{}{t}{-.2in}{\textbf{
            }}{%
                \includegraphics[width=0.9\textwidth]{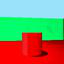}}
            \label{fig:smad7_reproduced_ensemble_ski}
        \end{subfigure}
    \begin{subfigure}{0.09\textwidth}
            \stackinset{c}{}{t}{-.2in}{\textbf{
            }}{%
                \includegraphics[width=0.9\textwidth]{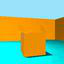}}
            \label{fig:smad7_reproduced_ensemble_ski}
        \end{subfigure}
    \begin{subfigure}{0.09\textwidth}
            \stackinset{c}{}{t}{-.2in}{\textbf{
            }}{%
                \includegraphics[width=0.9\textwidth]{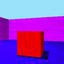}}
            \label{fig:smad7_reproduced_ensemble_ski}
        \end{subfigure}
    \begin{subfigure}{0.09\textwidth}
            \stackinset{c}{}{t}{-.2in}{\textbf{
            }}{%
                \includegraphics[width=0.9\textwidth]{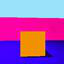}}
            \label{fig:smad7_reproduced_ensemble_ski}
        \end{subfigure}
    \begin{subfigure}{0.09\textwidth}
            \stackinset{c}{}{t}{-.2in}{\textbf{
            }}{%
                \includegraphics[width=0.9\textwidth]{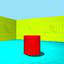}}
            \label{fig:smad7_reproduced_ensemble_ski}
        \end{subfigure}
    \begin{subfigure}{0.09\textwidth}
            \stackinset{c}{}{t}{-.2in}{\textbf{
            }}{%
                \includegraphics[width=0.9\textwidth]{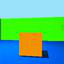}}
            \label{fig:smad7_reproduced_ensemble_ski}
        \end{subfigure}

    \rotatebox[origin=c]{90}{ Target\strut}
     \begin{subfigure}{0.09\textwidth}
            \stackinset{c}{}{t}{-.2in}{\textbf{}}{%
                \includegraphics[width=0.9\textwidth]{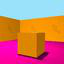}}
            \label{fig:smad7_reproduced_ensemble_Smad7mRNA}
        \end{subfigure}%
     \begin{subfigure}{0.09\textwidth}
            \stackinset{c}{}{t}{-.2in}{\textbf{
            }}{%
                \includegraphics[width=0.9\textwidth]{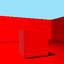}}
            \label{fig:smad7_reproduced_ensemble_Smad7}
        \end{subfigure}%
    \begin{subfigure}{0.09\textwidth}
            \stackinset{c}{}{t}{-.2in}{\textbf{
            }}{%
                \includegraphics[width=0.9\textwidth]{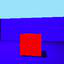}}
            \label{fig:smad7_reproduced_ensemble_ski}
        \end{subfigure}
    \begin{subfigure}{0.09\textwidth}
            \stackinset{c}{}{t}{-.2in}{\textbf{
            }}{%
                \includegraphics[width=0.9\textwidth]{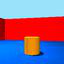}}
            \label{fig:smad7_reproduced_ensemble_ski}
        \end{subfigure}
    \begin{subfigure}{0.09\textwidth}
            \stackinset{c}{}{t}{-.2in}{\textbf{
            }}{%
                \includegraphics[width=0.9\textwidth]{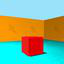}}
            \label{fig:smad7_reproduced_ensemble_ski}
        \end{subfigure}
     \begin{subfigure}{0.09\textwidth}
            \stackinset{c}{}{t}{-.2in}{\textbf{
            }}{%
                \includegraphics[width=0.9\textwidth]{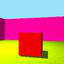}}
            \label{fig:smad7_reproduced_ensemble_ski}
        \end{subfigure}
    \begin{subfigure}{0.09\textwidth}
            \stackinset{c}{}{t}{-.2in}{\textbf{
            }}{%
                \includegraphics[width=0.9\textwidth]{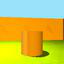}}
            \label{fig:smad7_reproduced_ensemble_ski}
        \end{subfigure}
    \begin{subfigure}{0.09\textwidth}
            \stackinset{c}{}{t}{-.2in}{\textbf{
            }}{%
                \includegraphics[width=0.9\textwidth]{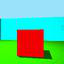}}
        \end{subfigure}
    \begin{subfigure}{0.09\textwidth}
            \stackinset{c}{}{t}{-.2in}{\textbf{
            }}{%
                \includegraphics[width=0.9\textwidth]{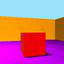}}
            \label{fig:smad7_reproduced_ensemble_ski}
        \end{subfigure}
    \begin{subfigure}{0.09\textwidth}
            \stackinset{c}{}{t}{-.2in}{\textbf{
            }}{%
                \includegraphics[width=0.9\textwidth]{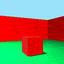}}
        \end{subfigure}
    \subcaption{Sshift 2}
    \vspace{-0.3cm}
        
\rotatebox[origin=c]{90}{Source\strut}
\vspace{-0.5cm}
    \begin{subfigure}{0.09\textwidth}
        \stackinset{c}{}{t}{-.2in}{\textbf{}}{%
            \includegraphics[width=0.9\textwidth]{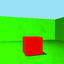}}

    \end{subfigure}%
    \begin{subfigure}{0.09\textwidth}
        \stackinset{c}{}{t}{-.2in}{\textbf{
        }}{%
            \includegraphics[width=0.9\textwidth]{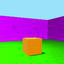}}
        \label{fig:smad7_reproduced_ensemble_Smad7}
    \end{subfigure}%
    \begin{subfigure}{0.09\textwidth}
        \stackinset{c}{}{t}{-.2in}{\textbf{
        }}{%
            \includegraphics[width=0.9\textwidth]{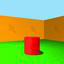}}
        \label{fig:smad7_reproduced_ensemble_ski}
    \end{subfigure}
    \begin{subfigure}{0.09\textwidth}
        \stackinset{c}{}{t}{-.2in}{\textbf{
        }}{%
            \includegraphics[width=0.9\textwidth]{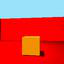}}
    \end{subfigure}
    \begin{subfigure}{0.09\textwidth}
        \stackinset{c}{}{t}{-.2in}{\textbf{
        }}{%
            \includegraphics[width=0.9\textwidth]{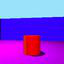}}
    \end{subfigure}
    \begin{subfigure}{0.09\textwidth}
        \stackinset{c}{}{t}{-.2in}{\textbf{
        }}{%
            \includegraphics[width=0.9\textwidth]{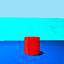}}
        \label{fig:smad7_reproduced_ensemble_ski}
    \end{subfigure}
    \begin{subfigure}{0.09\textwidth}
        \stackinset{c}{}{t}{-.2in}{\textbf{
        }}{%
            \includegraphics[width=0.9\textwidth]{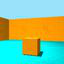}}
    \end{subfigure}
    \begin{subfigure}{0.09\textwidth}
        \stackinset{c}{}{t}{-.2in}{\textbf{
        }}{%
            \includegraphics[width=0.9\textwidth]{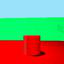}}
        \label{fig:smad7_reproduced_ensemble_ski}
    \end{subfigure}
    \begin{subfigure}{0.09\textwidth}
        \stackinset{c}{}{t}{-.2in}{\textbf{
        }}{%
            \includegraphics[width=0.9\textwidth]{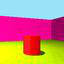}}
        \label{fig:smad7_reproduced_ensemble_ski}
    \end{subfigure}
    \begin{subfigure}{0.09\textwidth}
        \stackinset{c}{}{t}{-.2in}{\textbf{
        }}{%
            \includegraphics[width=0.9\textwidth]{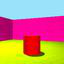}}
        \label{fig:smad7_reproduced_ensemble_ski}
    \end{subfigure}

\rotatebox[origin=c]{90}{ Target\strut}
    \begin{subfigure}{0.09\textwidth}
        \stackinset{c}{}{t}{-.2in}{\textbf{}}{%
            \includegraphics[width=0.9\textwidth]{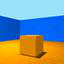}}
        \label{fig:smad7_reproduced_ensemble_Smad7mRNA}
    \end{subfigure}%
    \begin{subfigure}{0.09\textwidth}
        \stackinset{c}{}{t}{-.2in}{\textbf{
        }}{%
            \includegraphics[width=0.9\textwidth]{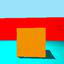}}
        \label{fig:smad7_reproduced_ensemble_Smad7}
    \end{subfigure}%
    \begin{subfigure}{0.09\textwidth}
        \stackinset{c}{}{t}{-.2in}{\textbf{
        }}{%
            \includegraphics[width=0.9\textwidth]{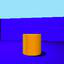}}
        \label{fig:smad7_reproduced_ensemble_ski}
    \end{subfigure}
    \begin{subfigure}{0.09\textwidth}
        \stackinset{c}{}{t}{-.2in}{\textbf{
        }}{%
            \includegraphics[width=0.9\textwidth]{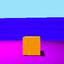}}
        \label{fig:smad7_reproduced_ensemble_ski}
    \end{subfigure}
    \begin{subfigure}{0.09\textwidth}
        \stackinset{c}{}{t}{-.2in}{\textbf{
        }}{%
            \includegraphics[width=0.9\textwidth]{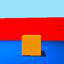}}
        \label{fig:smad7_reproduced_ensemble_ski}
    \end{subfigure}
    \begin{subfigure}{0.09\textwidth}
        \stackinset{c}{}{t}{-.2in}{\textbf{
        }}{%
            \includegraphics[width=0.9\textwidth]{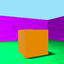}}
        \label{fig:smad7_reproduced_ensemble_ski}
    \end{subfigure}
    \begin{subfigure}{0.09\textwidth}
        \stackinset{c}{}{t}{-.2in}{\textbf{
        }}{%
            \includegraphics[width=0.9\textwidth]{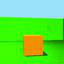}}
        \label{fig:smad7_reproduced_ensemble_ski}
    \end{subfigure}
    \begin{subfigure}{0.09\textwidth}
        \stackinset{c}{}{t}{-.2in}{\textbf{
        }}{%
            \includegraphics[width=0.9\textwidth]{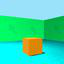}}
        \label{fig:smad7_reproduced_ensemble_ski}
    \end{subfigure}
    \begin{subfigure}{0.09\textwidth}
        \stackinset{c}{}{t}{-.2in}{\textbf{
        }}{%
            \includegraphics[width=0.9\textwidth]{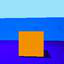}}
        \label{fig:smad7_reproduced_ensemble_ski}
    \end{subfigure}
    \begin{subfigure}{0.09\textwidth}
        \stackinset{c}{}{t}{-.2in}{\textbf{
        }}{%
            \includegraphics[width=0.9\textwidth]{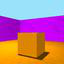}}
        \label{fig:smad7_reproduced_ensemble_ski}
    \end{subfigure}
        \subcaption{Dshift}
         \vspace{-0.3cm}
\rotatebox[origin=c]{90}{ Source\strut}
\vspace{-0.5cm}
    \begin{subfigure}{0.09\textwidth}
        \stackinset{c}{}{t}{-.2in}{\textbf{}}{%
            \includegraphics[width=0.9\textwidth]{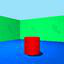}}
        \label{fig:smad7_reproduced_ensemble_Smad7mRNA}
    \end{subfigure}%
    \begin{subfigure}{0.09\textwidth}
        \stackinset{c}{}{t}{-.2in}{\textbf{
        }}{%
            \includegraphics[width=0.9\textwidth]{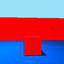}}
        \label{fig:smad7_reproduced_ensemble_Smad7}
    \end{subfigure}%
    \begin{subfigure}{0.09\textwidth}
        \stackinset{c}{}{t}{-.2in}{\textbf{
        }}{%
            \includegraphics[width=0.9\textwidth]{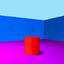}}
        \label{fig:smad7_reproduced_ensemble_ski}
    \end{subfigure}
    \begin{subfigure}{0.09\textwidth}
        \stackinset{c}{}{t}{-.2in}{\textbf{
        }}{%
            \includegraphics[width=0.9\textwidth]{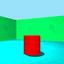}}
        \label{fig:smad7_reproduced_ensemble_ski}
    \end{subfigure}
    \begin{subfigure}{0.09\textwidth}
        \stackinset{c}{}{t}{-.2in}{\textbf{
        }}{%
            \includegraphics[width=0.9\textwidth]{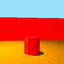}}
        \label{fig:smad7_reproduced_ensemble_ski}
    \end{subfigure}
    \begin{subfigure}{0.09\textwidth}
        \stackinset{c}{}{t}{-.2in}{\textbf{
        }}{%
            \includegraphics[width=0.9\textwidth]{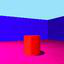}}
        \label{fig:smad7_reproduced_ensemble_ski}
    \end{subfigure}
    \begin{subfigure}{0.09\textwidth}
        \stackinset{c}{}{t}{-.2in}{\textbf{
        }}{%
            \includegraphics[width=0.9\textwidth]{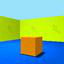}}
        \label{fig:smad7_reproduced_ensemble_ski}
    \end{subfigure}
    \begin{subfigure}{0.09\textwidth}
        \stackinset{c}{}{t}{-.2in}{\textbf{
        }}{%
            \includegraphics[width=0.9\textwidth]{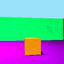}}
        \label{fig:smad7_reproduced_ensemble_ski}
    \end{subfigure}
    \begin{subfigure}{0.09\textwidth}
        \stackinset{c}{}{t}{-.2in}{\textbf{
        }}{%
            \includegraphics[width=0.9\textwidth]{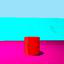}}
        \label{fig:smad7_reproduced_ensemble_ski}
    \end{subfigure}
    \begin{subfigure}{0.09\textwidth}
        \stackinset{c}{}{t}{-.2in}{\textbf{
        }}{%
            \includegraphics[width=0.9\textwidth]{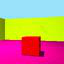}}
        \label{fig:smad7_reproduced_ensemble_ski}
    \end{subfigure}

\rotatebox[origin=c]{90}{ Target\strut}
    \begin{subfigure}{0.09\textwidth}
        \stackinset{c}{}{t}{-.2in}{\textbf{}}{%
            \includegraphics[width=0.9\textwidth]{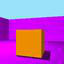}}
        \label{fig:smad7_reproduced_ensemble_Smad7mRNA}
    \end{subfigure}%
    \begin{subfigure}{0.09\textwidth}
        \stackinset{c}{}{t}{-.2in}{\textbf{
        }}{%
            \includegraphics[width=0.9\textwidth]{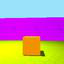}}
        \label{fig:smad7_reproduced_ensemble_Smad7}
    \end{subfigure}%
    \begin{subfigure}{0.09\textwidth}
        \stackinset{c}{}{t}{-.2in}{\textbf{
        }}{%
            \includegraphics[width=0.9\textwidth]{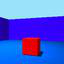}}
        \label{fig:smad7_reproduced_ensemble_ski}
    \end{subfigure}
    \begin{subfigure}{0.09\textwidth}
        \stackinset{c}{}{t}{-.2in}{\textbf{
        }}{%
            \includegraphics[width=0.9\textwidth]{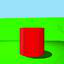}}
        \label{fig:smad7_reproduced_ensemble_ski}
    \end{subfigure}
    \begin{subfigure}{0.09\textwidth}
        \stackinset{c}{}{t}{-.2in}{\textbf{
        }}{%
            \includegraphics[width=0.9\textwidth]{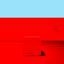}}
        \label{fig:smad7_reproduced_ensemble_ski}
    \end{subfigure}
    \begin{subfigure}{0.09\textwidth}
        \stackinset{c}{}{t}{-.2in}{\textbf{
        }}{%
            \includegraphics[width=0.9\textwidth]{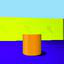}}
        \label{fig:smad7_reproduced_ensemble_ski}
    \end{subfigure}
    \begin{subfigure}{0.09\textwidth}
        \stackinset{c}{}{t}{-.2in}{\textbf{
        }}{%
            \includegraphics[width=0.9\textwidth]{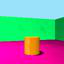}}
        \label{fig:smad7_reproduced_ensemble_ski}
    \end{subfigure}
    \begin{subfigure}{0.09\textwidth}
        \stackinset{c}{}{t}{-.2in}{\textbf{
        }}{%
            \includegraphics[width=0.9\textwidth]{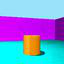}}
    \end{subfigure}
    \begin{subfigure}{0.09\textwidth}
        \stackinset{c}{}{t}{-.2in}{\textbf{
        }}{%
            \includegraphics[width=0.9\textwidth]{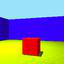}}
\end{subfigure}
    \begin{subfigure}{0.09\textwidth}
        \stackinset{c}{}{t}{-.2in}{\textbf{
        }}{%
            \includegraphics[width=0.9\textwidth]{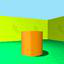}}
    \end{subfigure}
        \subcaption{Hshift}
        \vspace{-0.3cm}

\end{figure*}

\begin{table}[!htbp]
\small
\begin{tabular}{llll}
\toprule
                          & Factor & Source & Target \\
                          \midrule
\multirow{2}{*}{Sshift 1} & $\Pbb(Y,A)$     &  $[0.1,0.4,0.4,0.1]$&  same  \\ 
\cmidrule{2-4} 
                          & $\Pbb(D)$     & $[\frac{4}{16},\frac{4}{16},\frac{3}{16},\frac{1}{16},\frac{1}{16},\frac{1}{16},\frac{1}{16},\frac{1}{16}]$       &   $[\frac{1}{16},\frac{1}{16},\frac{1}{16},\frac{1}{16},\frac{1}{16},\frac{3}{16},\frac{4}{16},\frac{4}{16}]$       \\
                          \midrule
\multirow{2}{*}{Sshift 2} & $\Pbb(Y,A)$    & $[0.1,0.4,0.4,0.1]$       &    $[0.4,0.1,0.1,0.4]$    \\ 
\cmidrule{2-4} 
                          & $\Pbb(D)$      & $[\frac{1}{8},\frac{1}{8},\frac{1}{8},\frac{1}{8},\frac{1}{8},\frac{1}{8},\frac{1}{8},\frac{1}{8}]$       & same      \\
                          \midrule
\multirow{2}{*}{Dshift} & $\Pbb(Y,A)$     &  $[0.1,0.4,0.4,0.1]$&  same  \\ 
\cmidrule{2-4} 
                          & $\Pbb(D)$     & $[\frac{1}{2},\frac{1}{2},0,0,0,0,0,0]$       &   $[\frac{1}{8},\frac{1}{8},\frac{1}{8},\frac{1}{8},\frac{1}{8},\frac{1}{8},\frac{1}{8},\frac{1}{8}]$       \\
                          \midrule
\multirow{2}{*}{Hshift} & $\Pbb(Y,A)$     &  $[0.1,0.4,0.4,0.1]$&  $[0.4,0.1,0.1,0.4]$  \\ 
\cmidrule{2-4} 
                          & $\Pbb(D)$     & $[\frac{1}{2},\frac{1}{2},0,0,0,0,0,0]$       &   $[\frac{1}{8},\frac{1}{8},\frac{1}{8},\frac{1}{8},\frac{1}{8},\frac{1}{8},\frac{1}{8},\frac{1}{8}]$       \\
                          \bottomrule
\caption{Simulate different distribution shifts. $\Pbb(Y,A)$ is represented by the proportions of four groups as $[\Pbb(Y=0, A=0), \Pbb(Y=0, A=1), \Pbb(Y=1, A=0), \Pbb(Y=1, A=1)]$. $\Pbb(D)$ is represented by the proportions of eight possible values of \textit{scale}. Other factors have uniform marginal distributions. Images in two domains are sampled according to the marginal distributions of six latent factors.}
\label{tab:3dshape_shifts}
\end{tabular}
\end{table}

\subsection{Real Datasets}

\textbf{UTKFace}\footnote{https://susanqq.github.io/UTKFace/} \cite{zhifei2017cvpr} is a face dataset with images annotated with age, gender, and race. The data is collected from MORPH, CACD and Web. In our experiments, we use the aligned and cropped face images with ages larger than 10. 
We do gender classification which is a binary classification task, and set sensitive attribute to be the race. 
We consider binary-sensitive attribute case in our experiment by setting race to be white or non-white. 
The statistics of this dataset are shown in Table~\ref{tab:dataset_face}. 
\begin{table}[!htbp]
\small
\begin{tabular}{cccccc}
\toprule
\multicolumn{2}{c}{$(Y, A)$}            & (Male, White) & (Male, Black) & (Female, White) & (Female, Black) \\ 
\midrule
\multirow{2}{*}{UTK (S)}      & train &   3127            &   1508                & 2480                &     1450                \\ \cmidrule{2-6} 
                              & test  &    1377         &   651               &   1027            &      617              \\ 
                              \midrule
\multirow{2}{*}{FairFace (T)} & train &     7796            &   4650                   &      6946            &   5160                      \\ \cmidrule{2-6} 
                              & test  &    984           & 620                  &   839              &     635                \\ 
                             \bottomrule
\caption{Statistics of UTKFace and FairFace datasets.}
\label{tab:dataset_face}
\end{tabular}
\end{table}

\textbf{FairFace}\footnote{https://github.com/joojs/fairface} \cite{karkkainen2019fairface} is another large-scale face dataset with images annotated with age, gender, and race as well. Different from UTKFace, the data in FairFace is collected from Flickr, Twitter, Newspapers, and the Web. We also use images with ages larger than 10 in our experiments. We set the label to be the gender and the sensitive attribute to be the race. 
See statistics in Table~\ref{tab:dataset_face}. 
All face images in UTK and FairFace are resized to $96\times 96 \times 3$ in our experiments.

\textbf{NewAdult}\footnote{https://github.com/zykls/folktables} \cite{ding2021retiring} 
is a suite of datasets derived from US Census surveys. The data spans multiple years and all states of the United States which is a good fit for studying distribution shifts. In our experiments, we use 2018 data that span all states and do income classification with a threshold of 50,000 dollars. 
We set gender to be the sensitive attribute. 
We consider a problem that we train a fair classifier in California (source domain) and deploy it in other states (target domain).
The statistics are shown in Table~\ref{tab:dataset_newadult}. 
The input contains 10 features (see Appendix B.1 in \cite{ding2021retiring}) which are preprocessed to one-hot embeddings in our experiments.
\begin{table}[!htbp]
\small
\begin{tabular}{cccccc}
\toprule
\multicolumn{2}{c}{$(Y, A)$}            & (High, Male) & (High, Female) & (Low, Male) & (Low, Female) \\ 
\midrule
\multirow{2}{*}{CA (S)}      & train &      33258         &       22314            &     39224            &    42169                 \\ \cmidrule{2-6} 
                              & test  &   14839            &    9924           &    15990             &      17947                \\ 
                              \midrule
\multirow{2}{*}{Other states (T)} & train &  232162             &        140876           &      296826           &      351970               \\ \cmidrule{2-6} 
                              & test  &   101934           &      57798             &   127654              &   150544                  \\ 
                             \bottomrule
\caption{Statistics of NewAdult dataset.}
\label{tab:dataset_newadult}
\end{tabular}
\end{table}

\subsection{Experimental Settings}
\subsubsection{Experiments on 3dshapes}

\textbf{Model.} We use a two-layer MLP with 512 hidden units as the encoder and one linear layer as the classifier. The adversaries used in Laftr, CFair, and DANN are also two-layer MLP with 512 hidden units. ReLU is used as the activation function.

\textbf{Transformations.} We use random center cropping and padding as the transformation functions in consistency regularization. Such transformations can perfectly change the \textit{scale} of the objects to propagate labels from the source domain to the target domain.

\textbf{Setup.} We use SGD as the optimizer. We train every model with 200 epochs and select the best model according to the model's performance on the validation set. Base and Laftr only have access to the source data, and the model selection is based on the source validation set. For other methods that can access the unlabeled target data, the model selection is based on the labeled target validation set. 
Since \textit{accuracy} and \textit{fairness} are both important metrics, we select the best model according to the value of accuracy minus unfairness which is $Acc - \Delta_{odds}$. 
The coefficients of the fairness loss and consistency loss are both set to be 1.
We run every method five times and report the mean and the standard deviation.

\subsubsection{Experiments on UTK-FairFace}

\textbf{Model.} 
We use VGG16 and ResNet18 as the model in our experiments. The last linear layer is the classifier, and all the previous layers construct the encoder. When we use VGG16 as the model, we set every adversary used in Laftr, CFair, and DANN to be a two-layer MLP with 1024 hidden neurons. When ResNet18 is the model, the adversary has 512 hidden neurons.

\textbf{Transformations.}
We use RandAugment \cite{cubuk2020randaugment} as the transformation function which contains data augmentations that are the best for the CIFAR-10 dataset. To restrict the transformations to be group-preserving, we exclude augmentations that may change the color (so to change the race). 
The transformation function used in our experiments contains AutoContrast, Brightness, Equalize, Identity, Posterize, Rotate, Sharpness, ShearX, ShearY, TranslateX, and TranslateY. 
In this experiment, we use a weak augmented (with random cropping and flipping) image as the original input $\bx$ and a strong augmented (with RandAugment) image as the transformed input $t(\bx)$.

\textbf{Setup.} We use SGD as the optimizer. We train every model with 200 epochs and use the way introduced in the 3dshapes experiment to select the best model. 
The coefficients of the fairness loss and consistency loss are both set to be 1.
We run every method five times and report the mean and the standard deviation.

\subsubsection{Experiments on NewAdult}

\textbf{Model.} 
We use a 3-layer MLP with hidden sizes of (256, 512, 256) as the encoder and a 2-layer MLP with a hidden size of 128 as the classifier. Every adversary is a two-layer MLP with 128 hidden neurons.

\textbf{Transformations.} 
Studies on data augmentations for tabular data are very limited. 
In this paper, we use random corruptions on the input features as the transformation function. There are ten features in the input, and every time we only corrupt half of them. Additionally, for important factors that are highly correlated with the label, including the OCCP (occupation), COW (class of worker), we do not do any corruption. 
For factor SEX (gender), we do not do any corruption to preserve the group. 
For continuous factors including AGEP (age), SCHL (educational attainment), and WKHP (work hours), we do perturbations within a range. 
For other factors, we do uniformly sampling from their value spaces as corruptions. 
We do such transformations based on our assumption that they do not change the label. 
For example, two individuals that have five years of age gap but have the same other features should have similar income, and two individuals that only differ in the place of birth should have similar income. 
We admit that such transformations may not be the best ones. We need better domain knowledge on income prediction to design more powerful transformations. We leave the improvement of transformations for tabular data to future work.  

\textbf{Setup.} We use SGD as the optimizer.
We train every model with 200 epochs and use the metric introduced in the 3dshapes experiment to select the best model. 
The coefficients of the fairness loss and consistency loss are both set to be 1.
We run every method five times and report the mean and the standard deviation.

\subsection{Baselines}
\textbf{Laftr} is an adversarial learning method for algorithmic fairness. The adversary aims at accurately predicting the sensitive attribute based on the representation, while the encoder aims at making it hard. By adversarial learning, the representation will not contain information on sensitive attributes, so the prediction based on it will be fair. 
We denote the data in each group to be $\Dcal_a^y=\{\bx\in \Dcal | A=a, Y=y\}$. The fairness loss is designed to be
\begin{align*}
    L_{fair}=\sum_{(a,y)\in \{0,1\}^2}\frac{1}{|\Dcal_a^y|}\sum_{\bx \in \Dcal_a^y}|h(f(\bx))-a|.
\end{align*}
where $f$ is the encoder, $h$ is the adversary. \cite{david_learning_fair} prove that this loss is an upper bound of the equalized odds. The adversary minimizes this loss, while the encoder maximizes this loss with a gradient reversal layer.

\textbf{CFair} is similar to Laftr but uses two adversaries $h'$, and $h''$ for two classes with a balanced error rate (BER) defined as follows. We denote the data from one class to be $\Dcal^y=\{\bx\in \Dcal | Y=y\}$.
\begin{align*}
    L_{fair}=\text{BER}_{\Dcal^0}(h'(f(\bx)) \parallel A) +\text{BER}_{\Dcal^1}(h''(f(\bx)) \parallel A) 
\end{align*}

where $\text{BER}_{\Dcal^0}(h'(f(\bx)) \parallel A)=\frac{1}{2}\Pbb_{\Dcal^0 }(h'(f(\bx))\neq A|A=0)+\frac{1}{2}\Pbb_{\Dcal^0 }(h'(f(\bx))\neq A|A=1)$. 
In practice, the balanced error rate is estimated by the following cost-sensitive cross-entropy loss.

\begin{align*}
    \Pbb_{\Dcal^0 }(h'(f(\bx))\neq A|A=0)\leq \frac{\text{CE}_{\Dcal_0^0}(h'(f(\bx))\parallel A)}{\Pbb_{\Dcal^0 }(A=0)}
\end{align*}

\textbf{Laftr+FixMatch} uses the same framework as our method but with a standard consistency regularization that does not care about group performance. The consistency loss is defined as
\begin{align*}
    L_{consis}(g)=\frac{1}{|\Dcal|}\sum_{\bx\in \Dcal} \mathds{1}(\max(g_{tc}(\bx))\geq \tau)H(\argmax(g_{tc}(\bx)), g(t(\bx)))
\end{align*}
where $\Dcal$ denotes the entire dataset.
\subsection{Time and Space Complexity}
Compared with Base, Laftr, and CFair which only uses labeled source data, our method needs more training time and memory since we use unlabeled target data as well. 
Compared with other baselines that also use target data, such as Laftr+DANN, the time complexity of our method is comparable to theirs. Nevertheless, our method needs much fewer parameters than Laftr+DANN since it requires an adversary to do domain classification while we do not need it. 
Our method has the same model parameters as that in Laftr but with an additional consistency loss.

\section{More Experimental Results}\label{app:add-exp}

\subsection{Additional Results on UTKFace-FairFace with a Different Data Setting}
To evaluate our method in extreme circumstances, we conduct the UTKFace-FairFace experiment with less labeled source data and more unlabeled target data (see Table \ref{tab:dataset_face_app}). We also consider the race "white" and "non-white". Are shown in Table~\ref{tab:face-vgg_less}, we get consistent results that our method outperforms all baselines and can effectively transfer accuracy as well as fairness.
\begin{table}[!htbp]
\captionsetup{font=footnotesize}
\small
\centering
\caption{Statistics of UTK and FairFace datasets used in Table \ref{tab:face-vgg_less}.}
\resizebox{0.8\columnwidth}{!}{
\begin{tabular}{cccccc}
\toprule
\multicolumn{2}{c}{$(Y, A)$}            & (Male, White) & (Male, Non-white) & (Female, White) & (Female, Non-white) \\ 
\midrule
\multirow{2}{*}{UTK (S)}      & train &   1373            &   750                & 1650                &     1227                \\ \cmidrule{2-6} 
                              & test  &    565         &   285               &   614            &      370              \\ 
                              \midrule
\multirow{2}{*}{FairFace (T)} & train &     11429            &   16574                   &      8024            &   16838                      \\ \cmidrule{2-6} 
                              & test  &    1712           & 2453                  &   1176              &     2518                \\ 
                             \bottomrule

\label{tab:dataset_face_app}
\end{tabular}
}
\end{table}
\begin{table}[!htbp]
\captionsetup{font=footnotesize}
\small
\centering
\caption{Transfer fairness and accuracy from UTKFace to FairFace with less source data.} 
\resizebox{0.8\columnwidth}{!}{
\begin{tabular}{llllllll}
\toprule
              & \multicolumn{3}{c}{Source}            && \multicolumn{3}{c}{Target}            \\ 
              \cline{2-4} \cline{6-8} 
                              & Acc  & \multicolumn{2}{c}{Unfairness} && Acc & \multicolumn{2}{c}{Unfairness}\\
              \cline{3-4} \cline{7-8}
Method         &    & $V_{acc}$ & $\Delta_{odds}$ &  & & $V_{acc}$ & $\Delta_{odds}$ \\
\midrule
Base           & \facc{89.93}{0.43} &  \facc{2.79}{0.74}  & \facc{4.65}{0.44} && \facc{73.48}{0.56} & \facc{7.49}{3.50} & \facc{6.09}{1.07}   \\
Laftr           & \facc{90.61}{0.33} &  \facc{1.28}{0.43}  & \facc{3.62}{1.17} && \facc{73.29}{0.70} & \facc{5.42}{1.33} & \facc{7.78}{1.77}   \\
CFair          & \facc{90.68}{0.35} &  \facc{1.20}{0.59}  & \facc{3.61}{0.93} && \facc{73.82}{0.81} & \facc{5.71}{1.54} & \facc{7.37}{1.40}   \\
\midrule
Laftr+DANN      & \facc{90.53}{0.98} &  \facc{1.59}{0.97}  & \facc{4.62}{1.24} && \facc{74.44}{1.38} & \facc{6.94}{1.53} & \facc{10.26}{1.85}   \\
CFair+DANN      & \facc{90.23}{0.88} &  \facc{1.82}{0.97}  & \facc{4.96}{1.15} && \facc{74.53}{1.46} & \facc{9.27}{2.16} & \facc{9.96}{1.49}   \\

Laftr+FixMatch  & \facc{95.01}{0.10} &  \facc{1.37}{0.44}  & \facc{4.65}{1.00} && \facc{83.77}{0.45} & \facc{11.58}{1.16} & \facc{6.56}{1.74}   \\
CFair+FixMatch  & \facc{95.37}{0.24} &  \facc{1.13}{0.21}  & \facc{3.58}{0.90} && \facc{83.62}{0.51} & \facc{11.96}{1.05} & \facc{5.29}{1.76}   \\

Ours (w/ Laftr)  & \facc{94.77}{0.33} &  \facc{1.35}{0.70}  & \facc{3.28}{0.79} && \facc{84.65}{1.13} & \facc{2.92}{0.72} & \facc{6.99}{0.41}   \\

Ours (w/ CFair)  & \facc{94.92}{0.43} &  \facc{1.09}{0.30}  & \facc{3.00}{1.09} && \facc{84.71}{1.10} & \facc{3.57}{0.60} & \facc{7.34}{0.91}   \\
\bottomrule
\end{tabular}\label{tab:face-vgg_less}
}
\end{table}

\subsection{Additional Results on UTKFace-FairFace with Different Transformations}

\begin{table}[!htbp]
\captionsetup{font=footnotesize}
\small
\centering
\vspace{-0.5em}
\caption{Results by using different transformations in our method. Average results of three trials.} 
\resizebox{0.6 \columnwidth}{!}{
\begin{tabular}{llllllll}
\toprule
               & \multicolumn{3}{c}{Source}            && \multicolumn{3}{c}{Target}            \\ 
               \cline{2-4} \cline{6-8} 
                              & Acc  & \multicolumn{2}{c}{Unfairness} && Acc & \multicolumn{2}{c}{Unfairness}\\
               \cline{3-4} \cline{7-8}
Transformation         &    & $V_{acc}$ & $\Delta_{odds}$ &  & & $V_{acc}$ & $\Delta_{odds}$ \\
\midrule
None & 93.24 & 1.19 & 2.44 && 74.35 & 6.92 & 9.79 \\
All & 96.08 & 0.96 & 2.59 && 85.52 & 2.82 & 5.70 \\
AutoContrast   & 94.82 & 1.12 & 2.66  &  & 79.69 & 5.55 & 7.48  \\
Brightness     & 95.61  & 0.95 & 1.48  &  & 82.16 & 4.89 & 6.39   \\
Color          & 95.53 & 1.07 & 1.28   &  & 81.32 & 6.66  & 8.22  \\
Contrast       & 94.93 & 1.31  & 2.29  &  & 79.35 & 6.37 & 8.39  \\
Equalize       & 95.15 & 1.47 & 2.33  &  & 79.17 & 5.88 & 6.91  \\
Identity       & 96.21  & 1.03 & 1.31  &  & 81.58 & 3.44 & 7.29  \\
Posterize      & 94.92 & 1.77 & 3.06   &  & 79.63 & 5.26 & 6.01  \\
Rotate         & 96.13 & 0.72 & 1.83  &  & 84.33  & 3.80 & 6.34  \\
Sharpness      & 95.73 & 1.03 & 2.64  &  & 81.26 & 5.33 & 7.09  \\
ShearX         & 95.45 & 1.70 & 0.99  &  & 82.47 & 3.30  & 3.72  \\
ShearY         & 96.25 & 0.54 & 1.75  &  & 84.26 & 3.96 & 6.07  \\
Solarize       & 95.89 & 0.98 & 2.67  &  & 80.38  & 7.37 & 8.79  \\
TranslateX     & 96.11 & 0.89 & 1.79  &  & 83.49 & 2.31 & 6.13  \\
TranslateY     & 95.53 & 0.97 & 2.83  &  & 83.04 & 7.16 & 6.17  \\
\bottomrule
\end{tabular}\label{tab:face-aug}
}
\end{table}

To investigate the effect of different transformations in our method, we evaluate 14 transformations in RandAugment and report the results in Table 9. All the transformations can improve the accuracy in both domains. The effect on fairness varies. We find that \textit{Solarize}, \textit{Color}, and \textit{TranslateX} increase the unfairness in the source domain the most, and \textit{Contrast}, \textit{Color} and \textit{Solarize} have the highest unfairness in the target domain. Note that, it does not mean that these augmentations always lead to unfairness but that they are not suitable for our method. Recall that our theory and algorithm are built upon the intra-group expansion assumption. Transformations like \textit{Contrast}, \textit{Color}, and \textit{Solarize} may change the sensitive attribute "race" and break this assumption. Thus, in our experiments (Table~\ref{tab:face-vgg_rand}) we use all the transformations excluding \textit{Contrast}, \textit{Color}, and \textit{Solarize}.

\subsection{A Byproduct: Alleviate the Disparate Impact of Semi-supervised Learning}
\begin{figure}[b]
    \centering
    \includegraphics[width=0.3\textwidth]{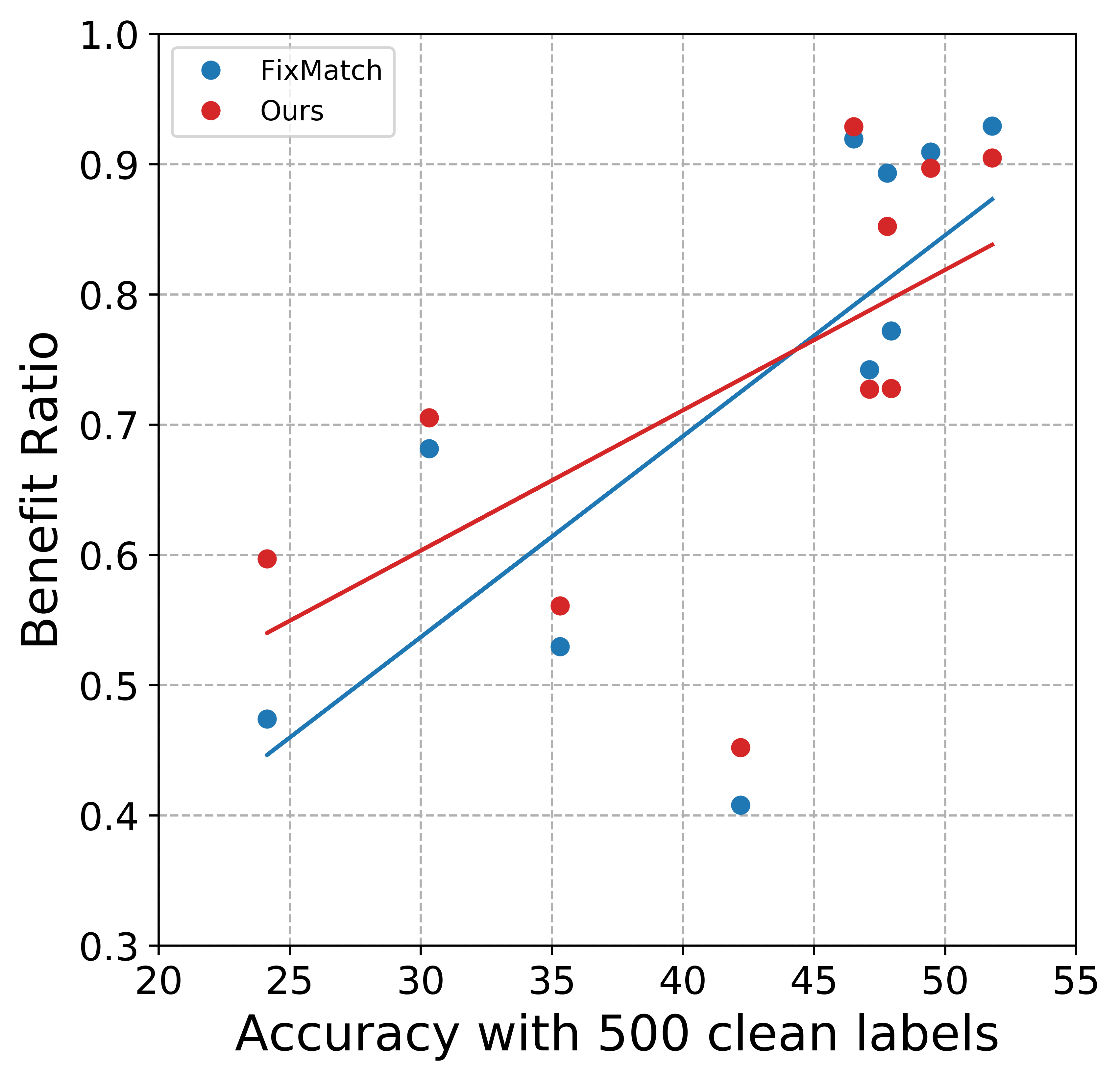}
    \caption{With fair consistency regularization, our method alleviates the disparate impact of FixMatch.}
    \label{fig:cifar}
\end{figure}
\cite{zhu2022the} find that semi-supervised learning methods may have a disparate impact. The classes that have high accuracy on labeled data are likely to benefit more from semi-supervised learning on unlabeled data. 
We test this argument on CIFAR-10 with FixMatch as the semi-supervised learning method. We use ResNet18 as the model. We randomly sample 500 images to be labeled data and treat others as unlabeled data. 
We use the benefit ratio proposed in \cite{zhu2022the} as the metric for the benefit of semi-supervised learning, defined as
\begin{align}
    BR(\Dcal)=\frac{a_{semi}(\Dcal)-a_{baseline}(\Dcal)}{a_{ideal}(\Dcal)-a_{baseline}(\Dcal)}.
\end{align}
where $\Dcal$ denotes the data from one class. $a_{semi}(\Dcal)$ is the model's test accuracy after semi-supervised learning, $a_{baseline}(\Dcal)$ is the test accuracy of the base model that is trained on labeled data, and $a_{ideal}(\Dcal)$ is the test accuracy of the ideal model where all data are labeled. We evaluate the benefit ratio of FixMatch on ten classes. As the blue line in Figure~\ref{fig:cifar} shows, the rich gets richer, and the poor gets poorer after semi-supervised learning. 
Our method (without using Laftr) can directly apply to this task. By using the proposed fair consistency regularization (red line in Figure~\ref{fig:cifar}), we can significantly improve the benefit ratio of the poor classes. Therefore, fair consistency regularization is a byproduct of this paper which is able to alleviate the disparate impact of semi-supervised learning.

\section{Impact and Limitations}\label{app:limit}
The fairness of machine learning is a critical problem in today's real-world applications. 
When distribution shifts happen, the collapse of fair systems will cause unexpected discrimination, resulting in severe negative social impacts. 
The fairness that is robust to distribution shifts is essential but is less explored. 
In this paper, the theoretical analysis of how fairness changes under different distribution shifts sheds light on the deep reasons for the collapse of fairness. 
The theory-guided self-training algorithm proposed in this paper explores a promising way to tackle distribution shifts. 
We hope our work will inspire more algorithms for this important and practical task. 

The major limitation of our method is that it strongly relies on pre-defined transformations as all the other self-training methods. 
The transformations are designed to be group-preserving based on our prior knowledge. 
Our experiments show that self-training with less powerful transformations has limited ability in propagating labels from source to target (i.e. transfer accuracy). 
Valid transformation functions on image data are thoroughly studied in existing work, while transformations on non-image data such as tabular data are much less explored. 
Our method with more powerful transformations on tabular data is expected to have significant improvement. Future work is encouraged to relax this limitation, such as by using a generative model as the transformation function.

\end{document}